\documentclass[10pt,journal,compsoc]{IEEEtran}

\ifCLASSOPTIONcompsoc
  \usepackage[nocompress]{cite}
\else
  \usepackage{cite}
\fi
\ifCLASSINFOpdf
\else
\fi
\usepackage{amsmath,amssymb, amsthm}
\usepackage{array}
\usepackage{stfloats}
\usepackage{subfig}
\usepackage{pgfplots}
\usepackage{booktabs}
\usepackage{ifthen}
\usepgfplotslibrary{units}

\newtheorem{lm}{Lemma}
\newtheorem{prop}{Theorem}
\newtheorem{cor}{Corollary}[prop]

\DeclareMathOperator\arctanh{arctanh}

\DeclareMathOperator\arcoth{arcoth}

\newcommand{\msg}[3][]{\mu^{#1}_{#2,#3}(x_#3)}
\newcommand{\coupling}[2]{J_{#1,#2}}
\newcommand{\messageNorm}[3]{\alpha_{#1,#2}^{#3}}
\newcommand{\field}[1]{\theta_{#1}}
\newcommand{\neighbors}[1]{\partial(X_{#1})}
\newcommand{\neighborsVariable}[2]{\partial(#1_{#2})} 
\newcommand{\neighborsWO}[2]{\Gamma_{{#1}\backslash{#2}}}
\newcommand{\totalPotentials}[2]{\Psi_{X_{#1},X_{#2}}\!(x_{#1},x_{#2})}
\newcommand{\map}[1]{\mathcal{BP}\!\{#1 \}}
\newcommand{\fb}{\mathbb{F}_B}
\newcommand{\fbP}{\mathbb{F}_B (P(X_i),P(X_i,X_j))}
\newcommand{\zb}{Z_B}
\newcommand{\zbApproximate}{\tilde{Z}_B}
\newcommand{\zbP}{Z_B \big(P(X_i),P(X_i,X_j)\big)}
\newcommand{\zbApproximateP}{\tilde{Z}_B \big(\tilde{P}(X_i),\tilde{P}(X_i,X_j)\big)}

\newcommand{\eqSys}{\mathbf{F}(\underline{\smash{\mu}}, \underline{\alpha})}
\newcommand{\startSys}{\mathbf{Q}(\underline{\smash{\mu}}, \underline{\alpha})}
\newcommand{\polytope}[1]{S_{#1}}
\newcommand{\polytopeLifted}[1]{\hat{S}_{#1}}
\newcommand{\magnetization}[1][]{\langle m_{#1} \rangle}
\newcommand{\setOfMarginals}[1]{\mathcal{\tilde{P}}_{\text{#1}}}

\newcommand{\attribute}[1]{{\ensuremath{Y_{#1}}}}
\DeclareMathOperator*{\argmax}{argmax}

\definecolor{mygreen}{rgb}{0.3, 0.73, 0.09}
\begin{document}
%
 \title{Fixed Points of Belief Propagation -- An Analysis via Polynomial Homotopy Continuation}
%
%
%
%

\author{Christian~Knoll,~\IEEEmembership{Student Member,~IEEE},
~Dhagash~Mehta,~ Tianran~Chen, and~Franz~Pernkopf,~\IEEEmembership{Senior Member,~IEEE}%
\IEEEcompsocitemizethanks{\IEEEcompsocthanksitem Christian Knoll (christian.knoll.c@ieee.org) and Franz Pernkopf (pernkopf@tugraz.at) are with the Signal Processing and Speech Communication Laboratory, Graz University of Technology.\protect\\
\IEEEcompsocthanksitem  Dhagash Mehta (Dhagash.B.Mehta.11@nd.edu) is with the Department of Applied and Computational Mathematics and Statistics, University of Notre Dame. \protect\\
\IEEEcompsocthanksitem  Tianran Chen (ti@nranchen.org) is with the Department of Mathematics and Computer Science, Auburn University at Montgomery.}
}

\IEEEtitleabstractindextext{%
\begin{abstract}
Belief propagation (BP) is an iterative method to perform approximate inference on arbitrary graphical models. Whether BP converges and if the solution is a unique fixed point depends on both the structure and the parametrization of the model. To understand this dependence it is interesting to find \emph{all} fixed points. In this work, we formulate a set of polynomial equations, the solutions of which correspond to BP fixed points. 

To solve such a nonlinear system we present the numerical polynomial-homotopy-continuation (NPHC) method.
Experiments on binary Ising models and on error-correcting codes show how our method is capable of obtaining all BP fixed points.
On Ising models with fixed parameters we show how the structure influences both the number of fixed points and the convergence properties.
We further asses the accuracy of the marginals and weighted combinations thereof.
Weighting marginals with their respective partition function increases the accuracy in all experiments.
Contrary to the conjecture that uniqueness of BP fixed points implies convergence, we find graphs for which BP fails to converge, even though a unique fixed point exists. Moreover, we show that this fixed point gives a good approximation, and the NPHC method is able to obtain this fixed point. 

\end{abstract}

}

\maketitle

\IEEEdisplaynontitleabstractindextext

%
\IEEEpeerreviewmaketitle

\section{Introduction}\label{sec:introduction}

%
%
%
%
\IEEEPARstart{J}{oint} distributions over many random variables (RVs) are often modeled as probabilistic graphical models. Belief propagation (BP) is a prominent tool to determine marginal distributions of 
such models. 
On tree-structured models the marginals are exact, but BP provides only an approximation on graphs with loops.
Despite the lack of guarantee for convergence, 
BP has been successfully used  for models with many loops, including 
applications in computer vision, medical diagnosis systems, and speech processing \cite{pearl1988,batra2012diverse,pernkopf2014pgm}.
However, instances of graphs do exist where BP fails to converge.  
A deeper understanding of the reasons for convergence of BP, and whether and how its non-convergence 
relates to the number of fixed points may therefore be crucial in understanding BP.

A precise relation among the uniqueness of fixed points, 
convergence rate, and accuracy is yet to be theoretically understood~\cite{mooij2007sufficient}.  Sufficient conditions for uniqueness of fixed points were refined by accounting for 
both the potentials as well as the structure of the model~\cite{heskes2004uniqueness,mooij2007sufficient}.
On graphs with a single loop~\cite{weiss2000correctness} and on small grid graphs~\cite{ihler07} accuracy and convergence rate are related; this does, however, not necessarily hold for all graphs\cite{weller2013approximating}.
As a consequence using provably convergent variants of BP can still lead to accurate approximations~\cite{murphy1999loopy, cccp2003yuille}.
Changing the BP update schedule can also help to achieve convergence \cite{elidan2012residual, knoll2015message, sutton2012improved}. 
One can further increase the accuracy of BP 
by considering multiple fixed points~\cite{batra2012diverse}. Survey propagation~\cite{braunstein2005survey} and its efficient approximation scheme~\cite{srinivasa2016survey} represent distributions over BP messages and marginalize over all obtained fixed points.

In this work we are interested in the relation among convergence properties, the number of fixed points, and the accuracy of these fixed points. 
To get deeper insights into the behavior of BP we aim to find \emph{all} fixed points -- including unstable ones.
If BP converges, however, it does only provide a \emph{single} fixed point.  In order to find all fixed points, we reformulate the update rules of BP as a system of polynomial equations.

There are indeed several methods to solve such a system of polynomial equations. 
Numerical solvers (e.g., Newton's method) are well established, 
but their ability in obtaining the solutions strongly depends on the initial point. Moreover, such methods 
only find a single solution at a time, and do not guarantee to find all solutions even with different initial guesses.
Symbolic methods, on the other hand, are guaranteed to find \emph{all} solutions. 
The Gr\"obner basis method~\cite{cox1992ideals,cox2006using}  
is widely used, but it is 
limited to systems with rational coefficients and it
suffers from fast growing run time and memory complexity. 
Moreover the method has a limited scalability in parallel computation.
In this paper, 
we describe the numerical polynomial homotopy continuation (NPHC) method \cite{li2003solving,sommese2005numerical} 
that overcomes all the above mentioned problems of both iterative and symbolic methods, yet guarantees to find 
\emph{all} solutions of the system.
Over the last few decades, this class of methods has been proven to be 
robust, efficient, and highly parallelizeable. 
We exploit the sparsity of our polynomial system and compute a tight upper bound on the number of complex solutions. This is an essential step to reduce the computational complexity by orders of magnitude so that the systems tackled in this work can be solved in practice.

We apply the NPHC method to different realizations of the Ising model on complete and grid graphs\footnote{A complete graph is an undirected graph where each pair of nodes is connected by an edge. A grid graph, or lattice graph, has all edges aligned along the 2D square lattice. Examples are depicted in Fig.~\ref{fig:graphStructure}.}, 
although it can be used for any other graph structure. 
These models are appealing because they are well studied in the physics literature~\cite{georgii, mezard2009}.
On these graphs we obtain all BP fixed points and show how the number of fixed points changes at critical regions (i.e., phase transitions) in the parameter space and how the fixed points behave under varying potentials. 

Recently we also performed stability analysis of all fixed points on Ising graphs with uniform potentials in~\cite{knoll2017stability}. 
We observed that the existence of non-vanishing local fields 
helps to achieve convergence and increases the accuracy of the best fixed point. We further showed that damping does not improve convergence properties on bipartite graphs.

We further use the obtained fixed points to estimate the approximation of the marginal distribution which we compare to the exact marginal distribution. 
Our main (empirical) observations are: 
(i) on loopy graphs BP does not necessarily converge to the best possible fixed point, 
(ii) on some graphs where BP does not converge we are able to show that there is a unique fixed point but if we enforce convergence to this unique fixed point the obtained marginals still give a good approximation,  and
(iii) convex combinations of multiple fixed points drastically improve the accuracy.
We further apply the NPHC method to analyze error-correcting code and show that BP performs better for densely connected variables.

The paper is structured as follows: Section \ref{sec:background} provides a brief background on probabilistic graphical models and free energy approximations. 
In Section \ref{sec:eqSys}, we reformulate the 
message passing equations, and introduce the NPHC method that guarantees to find all BP fixed point solutions. 
Our experimental results are presented and discussed in Section \ref{sec:experiments}. 
Finally, we conclude the paper in Section \ref{sec:conclusion}.

\section{Background} \label{sec:background}
In this section, we briefly introduce 
probabilistic graphical models and the BP algorithm,
and fix our notations. 
For an in-depth treatment of these topics we refer the reader to \cite{koller-friedman, jordan2004, pernkopf2014pgm}.

\subsection{Probabilistic Graphical Models} \label{subsec:PGM} 
We consider a finite set of $N$ discrete random variables $\mathbf{X} = \{X_1,\dots,X_N\}$ taking values from the binary alphabet $x_i \in \mathbb{S} = \{ -1, +1\}$. 
Let us consider the joint distribution $P(\mathbf{X} = \mathbf{x})$ and the corresponding undirected graph $G = (\mathbf{X},\mathbf{E})$ where $\mathbf{X} = \{X_1,\dots,X_N\}$ is the set of nodes and  $\mathbf{E}$ is the set of edges. 
Between the RVs and the nodes a one-to-one correspondence holds.
An interaction between two nodes  $X_{i}$ and $X_j$, $i\neq j$ is represented by an undirected edge $e_{i,j} \in \mathbf{E}$. We sometimes consider the graph $ G' = (\mathbf{X},\mathbf{E'})$ where $\mathbf{E'} \subset \mathbf{E}$ such that two nodes are connected by one edge at most, i.e.,  $e_{i,j} \in \mathbf{E'} \implies e_{j,i} \notin \mathbf{E'}$. Let the set of neighbors of $X_i$ be defined by $\neighbors{i} = \{X_j \in \mathbf{X} \backslash X_i : e_{i,j} \in \mathbf{E} \}$. 
The joint probability of an undirected graphical model factorizes to 
\begin{align}
P(\mathbf{X} = \mathbf{x}) = \frac{1}{Z} \prod_{l=1}^L \Phi_{C_l}(C_l),
\label{eq:factorization}
\end{align}
 where the potentials $\Phi_{C_l}$ are specified over the maximal cliques $C_l$ of the nodes \cite[p.105]{pearl1988}.
If we restrict all potentials to consist of at most two variables: then, the joint distribution is factorized according to
\begin{equation}
   P(\mathbf{X} = \mathbf{x}) = \frac{1}{Z} \prod_{(i,j)\,:\,e_{i,j} \in \mathbf{E'}} \Phi_{X_i,X_j}(x_i,x_j)\prod_{i=1}^N \Phi_{X_i}(x_i),
   \label{eq:joint}
\end{equation}
where $Z \in \mathbb{R_+^*}$ denotes the strictly positive normalization constant to guarantee a valid probability distribution.
The first product runs over all edges and the second product runs over all nodes; pairwise potentials and local evidence are denoted as $\Phi_{X_i,X_j}$ and $\Phi_{X_i}$ respectively.
We use a shorthand notation for the marginal probabilities $P(x_i) = P(X_i = +1)$ and $P(\bar{x}_i) = P(X_i = -1)$ where no ambiguities occur.

\subsection{Belief Propagation}\label{subsec:beliefPropagation}
Belief propagation (BP) 
is an algorithm that approximates marginal probabilities (or beliefs) $\tilde{P}(X_i = x_i)$. 
The marginals  are approximated by recursively updating messages between random variables. This update rule is guaranteed to converge and return the exact marginals on graphs without loops~\cite{pearl1988}. Note that this procedure has been discovered in different fields independently: belief propagation for probabilistic reasoning \cite{pearl1988}, the sum-product-algorithm in information theory \cite{gallager, kschischang2001factor}, and the Bethe-Peierls approximation in statistical mechanics \cite{mezard2009}.

The messages from $X_i$ to $X_j$ of state $x_j$ at iteration $n+1$ are given by the following update rule:
\begin{equation}
    \mu_{i,j}^{n+1}\!(x_j) \! = \! \messageNorm{i}{j}{n} \! \sum_{x_i \in \mathbb{S}} \!\!\Phi_{X_i,X_j} \!(x_i,x_j) \Phi_{X_i}(x_i) \!\!\!\!\!\prod_{X_k \in \neighborsWO{i}{j} } \!\!\!\! \mu_{k,i}^{n}(x_i), 
\label{eq:message}
    \end{equation}
where $\neighborsWO{i}{j} = \neighbors{i} \backslash \{X_j\}$. Loosely speaking BP collects all messages sent to $X_i$, except for $X_j$, and multiplies this product with the local potential $\Phi_{X_i}(x_i)$ and the pairwise potential $\Phi_{X_i,X_j}(x_i,x_j)$. The sum over both states of $X_i$ is sent to $X_j$. In practice the messages are often normalized by $\messageNorm{i}{j}{n} \in \mathbb{R}_+^*$ so as to sum to one \cite{ihler2005loopy}.
\begin{lm}\label{lm:msg}
 Messages being sent from node $X_i$ to $X_j$, $i \neq j$, represent probabilities -- provided all messages are initialized to be positive.
\end{lm}
\begin{proof}
 Positive potentials in \eqref{eq:message} guarantee that all messages remain positive at every iteration.
 Consequently, a normalization term $\messageNorm{i}{j}{n}$ exists so that $\sum\limits_{x_j \in \mathbb{S}} \mu_{i,j}^{n+1}(x_j) = 1$. 
\end{proof}

The set of all messages at iteration $n$ is given by $\underline{\smash{\mu}}^n = \{ \mu_{i,j}^n(x_j) : e_{i,j}  \in \mathbf{E} \}$. In a similar manner we collect all normalization terms in $\underline{\alpha}^n$. 
Let the mapping of all messages induced by \eqref{eq:message} be denoted as $\underline{\smash{\mu}}^{n+1} = \map{\underline{\smash{\mu}}^{n}} $.
If all successive messages show the same value (up to some predefined precision), that is $\underline{\smash{\mu}}^{n} \cong  \underline{\smash{\mu}}^{n+1}$, then
BP is converged. 
We refer to converged messages and the associated normalization terms as \emph{fixed points} $(\underline{\smash{\mu}}^*,\underline{\alpha}^*)$.

Note that at least one fixed point always exists if all potentials are positive \cite{yedidia2005constructing}. 
Existence of fixed points, however, is not sufficient to guarantee convergence; in fact BP may be trapped in limit cycles or show chaotic behavior~\cite{mackay2001conversation}.
If the messages oscillate, one can try to achieve convergence by damping~\cite{murphy1999loopy}, i.e., by replacing the messages with a weighted average of the last messages so that $\underline{\smash{\mu}}^{n+1} = (1-\epsilon) \map{\underline{\smash{\mu}}^{n}} +\epsilon\underline{\smash{\mu}}^{n}$: 
it follows, that any fixed point of BP with damping is a fixed point of BP without damping as well.

If BP converges to a fixed point, then the marginals are approximated by the normalized product
\begin{equation}
    \tilde{P}(X_i = x_i) = \frac{1}{Z_i} \Phi_{X_i}(x_i)\prod_{X_k\in \neighbors{i}} \mu^{*}_{k,i}(x_i),
    \label{eq:marginals}
\end{equation}
where $Z_i \in \mathbb{R_+^*}$ is required so that $\sum_{\mathbb{S}} P(X_i = x_i) = 1$.
Similar the pairwise marginal of two nodes $(X_i,X_j)$, connected by an edge, 
is defined by the product of all incoming messages times and all factors involved:
\begin{align}
\begin{split}
 \tilde{P}(X_i = x_i, X_j=x_j&) = \frac{1}{Z_{i,j}}  \totalPotentials{i}{j} \cdot \\
 & \prod_{X_k \in \neighborsWO{i}{j}}\msg[*]{k}{i}  \prod_{X_l \in \neighborsWO{j}{i}} \msg[*]{l}{j},
\end{split}
 \label{eq:pairwise}
\end{align}
where $\totalPotentials{i}{j} = \Phi_{X_i}(x_i) \Phi_{X_j}(x_j) \Phi_{X_i,X_j} \!(x_i,x_j)$ and $Z_{i,j}\in \mathbb{R_+^*}$ is the normalization term.

\subsection{Free Energy Approximations} \label{subsec:freeEnergy}
Over the years a fruitful connection between computer 
science and statistical mechanics was established (cf. \cite{mezard2009, welling2003approximate, tatikonda2002, weller2013approximating}).
In particular, the relationship between stationary points of the \emph{Bethe free energy} $\fb$ and 
fixed points of BP 
led to a deeper understanding of BP. 
We briefly discuss important insights 
and present differences in notations to circumvent any  confusions.

In the Ising model, a popular statistical mechanics model, often used for evaluation of BP, 
each node $X_i$ has an associated spin taking values in $\mathbb{S} = \{-1,+1\}$. 
We define the corresponding energy function~\cite[p.44]{mezard2009} by assigning a coupling $J_{i,j} \in \mathbb{R}$ to each edge $e_{i,j} \in \mathbf{E}$ and some local field $\theta_i \in \mathbb{R}$ acting on each node $X_i \in \mathbf{X}$. Note that we drop the subscripts of $J_{i,j}$ and $\theta_i$ whenever they are the same for all edges and nodes.
Let the local and pairwise Ising potentials of state $x_i$ be $\Phi_{X_i}(x_i) = \exp(\theta_i x_i)$ and $ \Phi_{X_i,X_j}(x_i,x_j) = \exp(J_{i,j}x_i x_j)$. 
Then by plugging these potentials into \eqref{eq:joint} the joint distribution is equal to the Boltzmann distribution
\begin{equation}
  P(\mathbf{X} = \mathbf{x}) = \frac{1}{Z} \exp \bigg( \beta \cdot \sum_{(i,j)\,:\,e_{i,j} \in \mathbf{E'}} J_{i,j} x_i x_j+ \sum_{i = 1}^{N}\theta_i x_i \bigg).
  \label{eq:ising}
\end{equation}
We omit the term of the inverse temperature $\beta$ by choosing $\beta = 1$ for the rest of this work.
The Bethe free energy is a function of the marginals and the pairwise marginals
\begin{align}
\begin{split}
\mathbb{F}_B &(P(X_i),P(X_i,X_j)) = \\ 
& \!\!\!\! \sum_{(i,j)\,:\,e_{i,j} \in \mathbf{E'}} \sum_{x_i,x_j} \!\!P(X_i = x_i, X_j=x_j) \cdot \\
&\ln \frac{P(X_i = x_i, X_j=x_j)}{\Phi_{X_i,X_j} \!(x_i,x_j)}  -\sum_{X_i} \sum_{x_i}P(X_i = x_i)\ln \Phi_{X_i}(x_i)\\
&-\sum_{X_i}(|\neighbors{i}|-1) \sum_{x_i} P(X_i = x_i)\ln P(X_i = x_i),
\end{split}
 \label{eq:betheFreeEnergy}
\end{align}
and relates to the \emph{Bethe partition function} according to
\begin{align}
\zbP = \exp\big(-\fbP \big).
\label{eq:bethePartition}
\end{align}
An excellent overview of free energy approximations from a variational perspective and how this relates to BP can be found in~\cite{yedidia2005constructing,wainwright2008graphical}.

%
%

If the Ising model is not on a path graph, critical regions in the parameter space can exist where phase transitions occur~\cite[Ch.12]{georgii}.
If all nodes $X_i \in \mathbf{X}$ have equal degree $|\neighbors{i}| = d+1$, these phase transitions can be determined by replacing the graph with a Cayley tree\footnote{A Cayley tree is an infinite tree without loops that captures interactions of a cyclic finite size graph.} of degree $d$~\cite{taga2004convergence}.
Therefore let 
\begin{align}
 &p(J,d) = \nonumber \\
&\begin{cases}
 d \arctanh{\sqrt{\frac{d\cdot w -1}{d/w -1}}} - \arctanh{\sqrt{\frac{d - 1/w}{d-w}}} \quad \text{if} \; J  > \arcoth(d)\\
 d \arctanh{\sqrt{\frac{d\cdot w -1}{d/w -1}}} + \arctanh{\sqrt{\frac{d - 1/w}{d-w}}} \quad \text{if} \; J  < \arcoth(d)
\end{cases}
\label{eq:phaseTransitions}
\end{align}
where $w = \tanh{|J|}$. Also, in accordance with \cite{taga2004convergence} and \cite[p.247-255]{georgii}, let the phase transitions partition the parameter space $(J,\theta)$ into the following three distinct regions $(I),(II)$, and $(III)$ as : 
\begin{align}
(J,\theta) &\in (I) \,\ \ \text{if} \ J  > 0, \ J > p(J,d)   \quad \text{and} \,|\theta| \leq p(J,d), \label{eq:phase1} \\
(J,\theta) &\in (II)\, \text{if}  \ J < 0, \ J < -p(J,d)  \ \text{and} \,|\theta| < p(J,d) , \label{eq:phase2}\\
(J,\theta) &\in (III)\,\text{if} \ (J,\theta) \notin (I) \quad\text{and} \quad  (J,\theta) \notin (II). \label{eq:phase3}
\end{align}

In the literature one distinguishes three different interactions on Ising grids. First, if all couplings $J > 0$ the model is \emph{ferromagnetic}: for ferromagnetic models BP converges to a unique stationary point inside $(III)$. This stationary point becomes unstable and two additional stationary points emerge inside $(I)$~\cite{yedidia2005constructing, mezard2009}.
Second, if all couplings $J < 0$ the model is \emph{antiferromagnetic}: here, BP only converges inside $(III)$ and not inside $(II)$~\cite{mooij2005properties}.
Finally, \emph{spin glasses} are models containing both positive and negative couplings.

Depending on the graph structure spin glasses and antiferromagnetic models allow for frustrations; i.e., the marginals minimizing the Bethe free energy may correspond to a degenerate joint distribution~\cite{mackay2001conversation},~\cite[pp.45]{mezard2009}.

\subsection{Stationary Points of Bethe Free Energy}
As stationary points of $\fb$ and of $\zb$ (cf.~\eqref{eq:bethePartition}) correspond to BP fixed points, one can try to obtain stationary points of $\fb$ directly, instead of performing BP.
For tree-structured graphs and one-loop graphs $\fb$ is a convex function~\cite{heskes2004uniqueness}.
For general graphs, however, convexity breaks down and $\fb$ has multiple local minima. Note that stable fixed points (cf. Sec.~\ref{sec:evolution}) of BP do always correspond to local minima of $\fb$; the converse however need not be the case, i.e., unstable fixed points can be either local maxima or local minima of $\fb$~\cite{heskes2003stable}.
Sufficient conditions for convexity of $\fb$ are often used to make statements regarding uniqueness of BP fixed point solutions.
This seems to be a rather limiting point of view -- it is possible to add loops to a former tree-structured model without changing the distribution \cite[pp.2391]{heskes2004uniqueness}.
Hence, the number of BP fixed points does not only depend on the structure of the graph but also depends on the potentials.

Energy functions with many local minima can be decomposed  into a convex and a concave problem (this decomposition is in general not unique)~\cite{cccp2003yuille}. Alternatively one can construct convex surrogates and minimize these convex functions instead~\cite{meshi2009convexifying}.
There is a variety of methods that obtain the marginal probabilities by minimizing $\fb$.
Belief optimization~\cite{welling2001belief} follows the negative gradient of $\fb$ and guarantees the marginalization constraints to be satisfied; i.e., the minimization takes place along the edges of the Bethe polytope.
An efficient way to obtain an approximate fixed point is presented in~\cite{shin2012complexity}. A projection scheme in the minimization task allows for a fully polynomial-time approximation on sparse graphs with $\max(|\neighbors{}|) = \mathcal{O}(\log N)$.
This algorithm is further improved in~\cite{weller2013approximating} to return an approximation of the best (stable) fixed point of BP.
An alternative method to approximate and combine all stable solutions of BP (i.e., local minima of $\fb$) is presented in~\cite{srinivasa2016survey}.

However all of these methods rely on the idea of finding minima of $\fb$; and consequently fail to obtain (unstable) fixed points that correspond to local maxima of $\fb$.


\section{Solving BP Fixed Point Equations}\label{sec:eqSys}
We reformulate the message update rules~\eqref{eq:message} as a system of polynomial equations whose solutions are the fixed points of BP.
Whether such systems can be solved in practice depends largely on the chosen method. 
We list a variety of approaches and describe NPHC that can be applied in practice to obtain \emph{all} BP fixed points.

\subsection{Reformulation of Belief Propagation}

For all messages $\mu_{i,j}(x_j) : e_{i,j} \in \mathbf{E}$ the residual, i.e., the difference between two successive message values $\underline{\smash{\mu}}^{n+1} - \underline{\smash{\mu}}^{n}$ and the  message normalization constraints $ \underline{\alpha}^n$ are given by the following system of polynomial equations:

\begin{align}
 &\eqSys = \nonumber \\ 
&\begin{cases}
 \! -\mu_{i,j}^{n}(x_j)      \!+\! \alpha_{i,j}^n \!\! \sum\limits_{x_i \in \mathbb{S}} \!\! \Phi_{X_i,X_j}(x_i,x_j)      \Phi_{X_i}(x_i) \!\!\!\prod\limits_{X_k \in \neighborsWO{i}{j}}\!\!\!\! \mu_{k,i}^{n}(x_i)\\ 
 \! -\mu_{i,j}^{n}(\bar{x}_j)\!+\! \alpha_{i,j}^n \!\! \sum\limits_{x_i \in \mathbb{S}} \!\!\Phi_{X_i,X_j}(x_i,\bar{x}_j) \Phi_{X_i}(x_i) \!\!\!\prod\limits_{X_k \in \neighborsWO{i}{j}}\!\!\!\! \mu_{k,i}^{n}(x_i)\\
 \mu_{i,j}^{n}(x_j) + \mu_{i,j}^{n}(\bar{x}_j) -1.
\end{cases}
\label{eq:SetOfEq}
\end{align}
This system of polynomial equations consist of 
\begin{align}
M = 2|\mathbf{E}| \cdot (|\mathbb{S}| + 1)
\label{eq:NrEq}
\end{align}
equations $(f_1(\underline{\smash{\mu}}, \underline{\alpha}), \ldots , f_M(\underline{\smash{\mu}}, \underline{\alpha}))$.
To solve such a polynomial system, it is advantageous to consider it defined over complex variables rather than 
real variables, in order to apply the NPHC method. Let the set of solutions over complex variables, without accounting for multiplicity, be 
\begin{align}
 V(\mathbf{F}) = \{ (\underline{\smash{\mu}}, \underline{\smash{\alpha}}) \in \mathbb{C} : f_m(\underline{\smash{\mu}},\underline{\alpha}) = 0 \text{ for all }f_m \in \eqSys\}.
\end{align}
We further define the set of solutions $V_{\mathbb{R}_+}^{*}(\mathbf{F}) \subset V(\mathbf{F})$ over strictly positive real 
numbers.

\begin{prop}[Fixed Points of BP]\label{prop:EqSys}
Let $(\underline{\smash{\mu}}, \underline{\alpha})$ be some set of messages and normalization terms. 
Then, $(\underline{\smash{\mu}}, \underline{\alpha})$ is a fixed point solution of BP, if and only if $(\underline{\smash{\mu}}, \underline{\alpha}) \in V_{\mathbb{R}_+}^{*}(\mathbf{F})$.
\end{prop}
\begin{proof}
First we show that $(\underline{\smash{\mu}}, \underline{\alpha}) \in V_{\mathbb{R}_+}^{*}(\mathbf{F})$ is sufficient to characterize fixed point solutions.
All messages are positive and represent probabilities (Lemma~\ref{lm:msg}). Further, it follows from \eqref{eq:SetOfEq} that $BP  \{ \underline{\smash{\mu}}^n \} - \underline{\smash{\mu}}^n = 0$, which constitutes a fixed point solution (cf. Section~\ref{sec:background}).

Conversely, consider some fixed point messages and the corresponding normalization coefficients $(\underline{\smash{\mu}}^* , \underline{\alpha}^*)$, it then follows by definition that $BP \! \{ \underline{\smash{\mu}}^* \} = \underline{\smash{\mu}}^*$ and consequently $\eqSys = 0$.
\end{proof}

\begin{cor}\label{prop:nonEmpty}
Consider a graph with strictly positive potentials. Then, the solution set $V_{\mathbb{R}_+}^{*}(\mathbf{F})$ is nonempty. Moreover, if $\Phi_{X_i,X_j}$ and $\Phi_{X_i}$ are Ising potentials, then $V_{\mathbb{R}_+}^{*}(\mathbf{F})$ is always nonempty.
\end{cor}
\begin{proof}
 For non-negative potentials the average energy is bounded from below and $\fb$ has at least one minimum \cite[Th. 4]{yedidia2005constructing}.
 Minima of the constrained $\fb$ correspond to BP fixed point solutions, the existence of which implies non-emptiness of 
 $V_{\mathbb{R}_+}^{*}(\mathbf{F})$ by Theorem \ref{prop:EqSys}.
\end{proof}

\subsection{Solving Systems of Polynomial Equations}

Solving systems of (nonlinear) polynomial equations is a classic problem 
in computational mathematics, and a great variety of approaches have
been developed such as iterative solvers, symbolic methods, and homotopy methods.

One basic method for solving systems of nonlinear equations is Newton's method 
which is an iterative solver that progressively refines an initial guess to 
reach a solution.
Such iterative solvers find a \emph{single} solution
in the vicinity of  an already known initial guess.
However, when the initial guess is not sufficiently close to a solution, 
they may diverge or even exhibit chaotic behavior.
Moreover, it is difficult, to obtain the full set of solutions with
these methods.
We are interested in the entire set 
of (positive) real solution $V_{\mathbb{R}_+}^{*}(\mathbf{F})$, i.e., in obtaining \emph{all} fixed points;
therefore iterative solvers are not useful in the current setting.

From a completely different point of view, symbolic methods \cite{cox1992ideals,cox2006using}
(e.g. Gr\"obner basis method, Wu's method, and methods of sparse resultant)
rely on symbolic manipulation of the equations and successive elimination of
variables to obtain a simpler but equivalent form of the equations.
In a sense, these methods can be considered as generalizations
of the Gaussian elimination method for linear systems into nonlinear settings.
In the past several decades, symbolic methods, especially the Gr\"obner basis method,
have seen substantial development.
But the method is limited to systems with rational coefficients and has a worst case complexity that is double exponential in the variables \cite{mayr1982complexity, moral1984upper}. 
This and the limited scalability in parallel computation limits the application of symbolic methods to smaller systems. In fact, the Gr\"obner did not converge for any of our experiments.

\subsection{Numerical Polynomial Homotopy Continuation}
Another important approach for solving a system of nonlinear equations is the
\emph{numerical polynomial homotopy continuation} (NPHC) method~\cite{li2003solving,sommese2005numerical}. 
The ``target'' system $\eqSys$ in~\eqref{eq:SetOfEq}, which we intend to solve, is continuously deformed into a 
closely related ``starting system'' $\startSys = ( q_1(\underline{\smash{\mu}}, \underline{\alpha}), \ldots , q_M(\underline{\smash{\mu}}, \underline{\alpha}) )$ 
that is trivial to solve. 
With an appropriate construction, the corresponding solutions also vary 
continuously under this deformation forming ``solution paths'' that connect the 
solutions of the trivial system to the desired solutions of the target system.

For instance, a basic form of linear homotopy for the target system
can be given by
\begin{align}
  H(\underline{\smash{\mu}}, \underline{\alpha}, t) = (1-t)\startSys + \gamma t\eqSys = 0.
\label{eq:linearHomotopy} 
\end{align}
Clearly, at $t=0$ the homotopy reduces to the starting system $\startSys$
and at $t=1$ it reduces to $\eqSys$. 
As $t$ varies continuously from $t=0$ to $t=1$,
the homotopy represents a deformation from the starting system to the target
system. 
For a generic complex $\gamma$, the above procedure is guaranteed to 
find \textit{all} isolated complex solutions of $\eqSys$ \cite[pp.91]{sommese2005numerical}.

Even though only positive real solutions $V_{\mathbb{R}_+}^{*}(\mathbf{F})$ are of interest in the current study, the homotopy method benefits from extending the domain to the field of complex numbers. Only then smooth solution paths emerge towards the desired solutions.
Each of these solution paths can be tracked independently making this approach \emph{pleasantly parallelizeable}; this is essential in dealing with large polynomial systems.

\subsection{Polyhedral Homotopy}
More advanced "nonlinear" homotopies can be constructed where the parameter
$t$ appears in nonlinear form. 
Among a great variety of polynomial homotopy constructions, the 
\emph{polyhedral homotopy method}, developed by B. Huber and B. Sturmfels~\cite{huber1995polyhedral}, 
is particularly suited
as it is capable of finding \emph{all} isolated nonzero complex solutions\footnote{Here, ``nonzero complex solutions'' refer to complex solutions of a system of polynomial equations where each variable is nonzero. A solution is considered to be isolated if it has no degree of freedom, i.e., there is an open set containing it but no other solutions.},
which must include \emph{all} BP fixed points $V_{\mathbb{R}_+}^{*}(\mathbf{F})$.
This advantage and the level of parallel scalability are the main motivation for applying the polyhedral homotopy method

In applying the NPHC method to solve \eqref{eq:SetOfEq},
the choice of $\startSys$ (the trivial system of equations that 
the target system is deformed into) plays an important role in the overall
efficiency of the approach since different choices of $\startSys$ may induce a vastly different
number of solution paths one has to track. The crucial part is to come up with a good upper bound on the number of solutions and to create an appropriate start system. Once this is solved, every solution can be tracked completely independently in parallel in order to reach all desired solutions.

In each equation of~\eqref{eq:SetOfEq} only few of the monomials are present, i.e., the update equations of BP imply a sparse system of equations~\cite{huber1995polyhedral}. This holds even if the graph is non-sparse.
In our experiments, we observed that despite the rather high
\emph{total degree}\footnote{%
    The total degree of a system of polynomial equations is the product of 
    the degrees of each equation.
    It is a basic fact in algebraic geometry that the total number of isolated complex solutions
    a polynomial system has is bounded by its total degree (i.e., Bezout bound).
    Therefore the total degree serves as a crude measure of the complexity
    of the polynomial system.} $d_t$~\cite[pp.118]{sommese2005numerical},
each equation in~\eqref{eq:SetOfEq} contains only few of the monomials   
the number of solution paths to be tracked 
to solve \eqref{eq:SetOfEq} is relatively small.
The number of solution paths one has to track when using the 
polyhedral homotopy method for solving a system of polynomial equations is given by the so-called \emph{Bernstein-Kushnirenko-Khovanskii (BKK) bound}: 
fixing the list of monomials that appear in the polynomial system, it is an important yet surprising fact in algebraic geometry 
that for almost all choices of the coefficients (in the probabilistic sense), the number of isolated
nonzero complex solutions is a fixed number which only depends on the list of monomials. 
This number is known as the BKK bound \cite{Bernstein75,Kushnirenko76,Khovanski78}. Intermediate steps in the determination of the BKK bound  are reused to create an appropriate start system. Using the fully parallel implementation \textsf{Hom4PS-3}~\cite{chen2014hom4ps} of the polyhedral homotopy method, we compute the BKK bound, that is tight in all our experiments, and obtain \emph{all} isolated positive solutions.

The polyhedral homotopy method exploits the structure of~\eqref{eq:SetOfEq}, but also requires some subtle steps. Rather than presenting all technical details we present an illustrative example to explain the underlying principles. For more details we refer the reader to the excellent overview papers \cite{li1997numerical,li2003solving, chen_homotopy_2015} or to \cite[Sec.8.5.4]{sommese2005numerical} and the references therein.
\subsubsection{Illustrative Example}

The essential steps in solving polynomial system with the polyhedral homotopy method are: first, compute a root count based on mixed volume computations~\cite[Section 3]{li2003solving}; secondly, come up with an easy to solve start system~\cite[Section 4]{li2003solving}; and finally, solve the start system and track the solution paths to the target system~\cite[Section 1]{li2003solving}.\\ \\
\emph{(i) Root Count:} 
Consider the example system in two unknowns $\underline{x} = \{x_1,x_2\}$ and with 4 solutions taken from~\cite[p.142]{sommese2005numerical}:
\begin{align}
 \mathbf{F}(\underline{x}) = 
&\begin{cases}
1+ax_1+bx_1^{2}x_2^{2}\\
1+cx_1+dx_2+ex_1x_2^{2}.
\end{cases} \label{eq:ExampleSystem}
\end{align}
The total degree of this system of equations is $d_t = 4\cdot 3  = 12$, which serves as an upper bound on the actual number of solutions. 
If the system of equations is sparse, the BKK bound serves
as a much tighter bound.

Every equation $f_m \in \mathbf{F}(\underline{x})$ has an associated polytope $\polytope{m}$ which is the convex hull of the exponent vectors for all monomials of $f_m$. For $f_1 $ the polytope is
\begin{align}
 \polytope{1} = \{(0,0)(1,0)(2,2)\},
\end{align}
which has a graphical representation in Fig.~\ref{fig:q1}. Similar for $f_2$ the polytope $\polytope{2}=\{(0,0)(1,0)(0,1)(1,2)\}$ is shown in Fig.~\ref{fig:q2}.


Some important operation on polytopes are: the Minkowski sum $\polytope{1}+\polytope{2} = \{s_1+s_2 : s_1 \in \polytope{1}, s_2 \in \polytope{2}\}$, and the computation of volumes, denoted by $V(\polytope{m})$. 
Note that this approach generalizes to higher dimensions, therefore we refer to $V(\polytope{m})$ as volume even if the polytopes lie in the two-dimensional space.
The computation of the mixed volume $M(\polytope{1},\polytope{2})$ is a combinatorial problem that is especially comprehensible in the case of two equations where
\begin{align}
 M(\polytope{1},\polytope{2}) = V(\polytope{1}+\polytope{2}) - V(\polytope{1})-V(\polytope{2}).
\end{align}
For a generalization to higher dimensions see~\cite[p.140]{sommese2005numerical}.
The polytope of $\polytope{1}+\polytope{2}$ is illustrated in Fig.~\ref{fig:q1+q2}: the mixed volume is then obtained by subtracting $V(\polytope{1})$ and $V(\polytope{2})$ from $V(\polytope{1}+\polytope{2})$,
which equals the sum of all gray areas (known as mixed cells) in Fig.~\ref{fig:q1+q2}. 

It is straightforward to see that each parallelogram has volume equal to 2;
consequently the BKK bound equals $M(\polytope{1},\polytope{2}) = 4$ and provides a tight bound on the number of solutions.

\begin{figure}
 \newcommand{\unitBKK}{1cm}
 \centering
 \subfloat[][]{ \label{fig:q1}
 \includegraphics[width =0.22\linewidth]{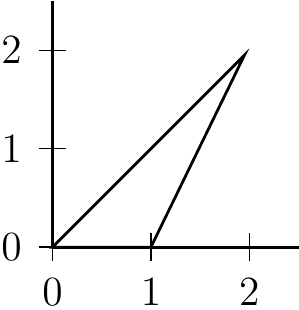}}
 \subfloat[][]{ \label{fig:q2}
 \centering
 \includegraphics[width =0.22\linewidth]{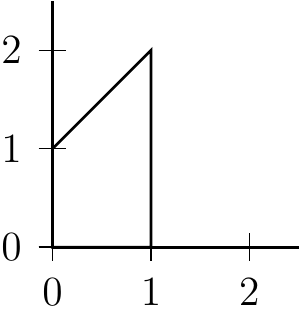}}
 \subfloat[][]{ \label{fig:q1+q2}
 \centering
 \includegraphics[width =0.3\linewidth]{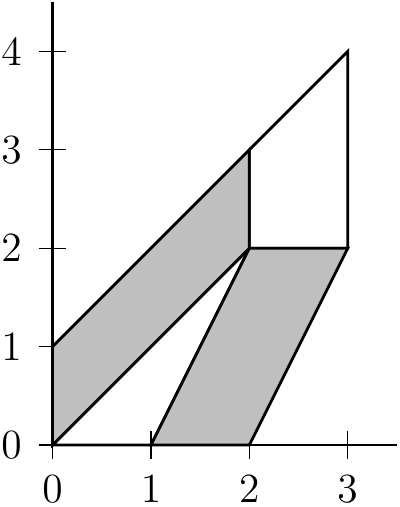}}
\caption[]{\subref{fig:q1} Polytope $\polytope{1}$, \subref{fig:q2} polytope $\polytope{2}$, and \subref{fig:q1+q2} Minkowski sum $\polytope{1}+\polytope{2}$}
\end{figure}

\emph{(ii) Start System:} The BKK-bound does not immediately lead to the solutions of an appropriate start system $\mathbf{Q}(\underline{x})$ with $q_m(x) = \sum_{a\in \polytope{m}} c_{m,a}\underline{x}^a$, where  $\underline{x}^a = x_1^{a_1}\cdot x_2^{a_2}$ and $c_{m,a}$ are random coefficients.

However, the mixed volume computation can also be accomplished by introducing a lifting $\omega_m = \{\omega_m(a):a \in \polytope{m}\}$ for each $f_m$.  
Thereby we increase the dimension of the polytope $\polytope{m}$ to $\polytopeLifted{m}$ by adding one component to each exponent-vector $a$. This component is obtained by the lifting function $\omega_m(a)$. In our example we choose the lifting values $\omega_1 = \{0,0,0\}$ and $\omega_2 = \{0,1,1,3\}$; these values are obtained by the inner products $\omega_1(a) = (0,0)\circ(a_1,a_2)$ and $\omega_2(a) = (1,1)\circ(a_1,a_2)$.
The polytopes are lifted accordingly so that $\polytopeLifted{1} = \{(0,0,0)(0,1,0)(2,2,0)\}$, $\polytopeLifted{2}=\{(0,0,0)(1,0,1)(0,1,1)(1,2,3)\}$ and 
$\polytopeLifted{1}+\polytopeLifted{2}=\{(0,0,0)(0,1,0)(2,2,0)(0,2,1)(0,1,1)(3,2,1)(3,4,3)\}$. Then the faces in the lower hull of $\polytopeLifted{1}+\polytopeLifted{2}$ correspond to cells shown in Fig.~\ref{fig:q1+q2}, which is known as a fine mixed subdivision.


These liftings and the random  coefficients $c_{m,a}$ form the homotopy $\hat{\mathbf{Q}}(\underline{x},t)$ with $\hat{q}_m =\sum_{a\in \polytope{m}} c_{m,a}\underline{x}^a t^{\omega_m(a)}$ such that 
\begin{align}
 \hat{\mathbf{Q}}(\underline{x},t) = 
&\begin{cases}
1+c_{1,1}x_1+c_{1,2}2x_1^{2}x_2^{2}\\
1+c_{2,1}x_1t+c_{2,2}x_2t+c_{2,3}x_1x_2^{2}t^3.
\end{cases}
\end{align}

By closer inspection, however, it is not possible to identify the starting points because $\hat{q}_2(\underline{x},t=0) = 1$. This problem can be resolved according to~\cite[Lemma 3.1]{huber1995polyhedral}: i.e., initial values are obtained by solving a binomial system for every cell that contributes to the mixed volume computation (gray cells in Fig.~\ref{fig:q1+q2}).
One can increase $t$ and track the solution paths to $\hat{\mathbf{Q}}(\underline{x},t=1) = \mathbf{Q}(\underline{x})$.

\emph{(iii) Target System:} Finally we have all 4 solutions to $\mathbf{Q}(\underline{x})$. Now what remains is to construct a \emph{linear} homotopy according to~\eqref{eq:linearHomotopy} and increase $t$ starting $t=0$. At $t=1$ the homotopy reduces to $\mathbf{F}(\underline{x})$ and provides the desired solutions.

\section{Experiments}\label{sec:experiments}

In various experiments we apply the NPHC method to system \eqref{eq:SetOfEq} and obtain all 
fixed points by first finding all the isolated non-zero complex solutions. 
We consider ferromagnetic, antiferromagnetic, and spin glass models on fully connected graphs and grid graphs shown in Fig.~\ref{fig:graphStructure}. We evaluate and compare the accuracy of all fixed points obtained by NPHC; details for our evaluation criteria are presented in Sec.~\ref{subsec:evaluationCriteria}.
Further we present the evolution of the fixed points over the parameter space in Sec.~\ref{sec:evolution} and present the implications on the accuracy to better understand under which settings BP can be expected to provide good results. 
The capability of NPHC to obtain all fixed points allows for a thorough stability analysis. This inspired some theoretical investigations on Ising graphs with uniform parameters~\cite{knoll2017stability}. There we show why convergence properties degrade with growing graph size and why non-vanishing fields help to achieve convergence. The limited influence of damping on bipartite graphs is also explained.

\begin{figure}[!t]
  \centering
  \subfloat[][]{ \label{fig:3x3} 
    \includegraphics[width =0.45\linewidth]{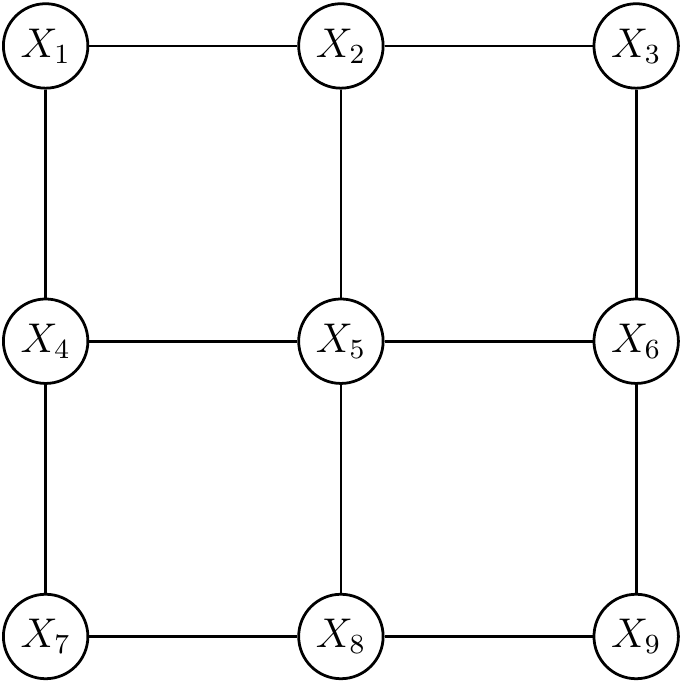}}
  \subfloat[][]{ \label{fig:2x2}
    \includegraphics[width =0.45\linewidth]{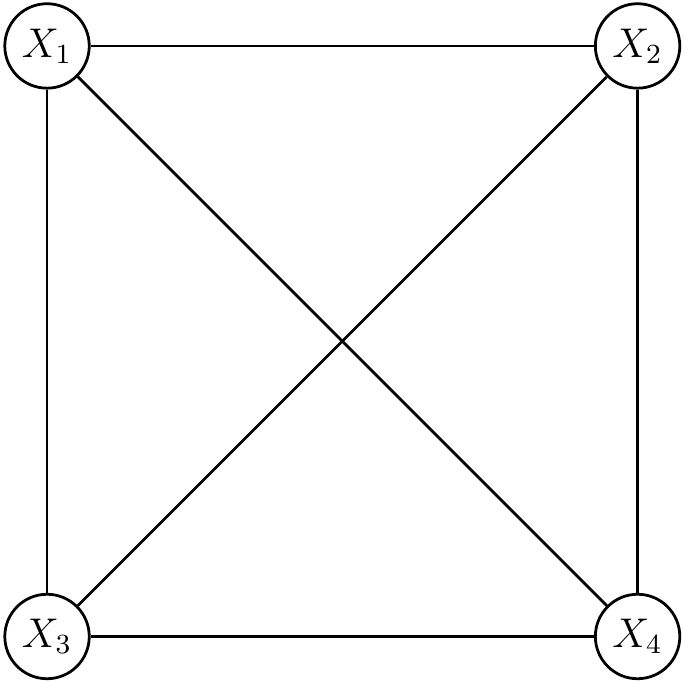}}
    \caption{\label{fig:graphStructure}Structure of Ising graphs considered: \protect\subref{fig:3x3} grid graph with 9 RVs;  \protect\subref{fig:2x2} fully connected graph.}
\end{figure}

The systems considered in this work are simply too large to be solved with symbolic methods or even with the NPHC method based on a linear homotopy. 
The BKK bound, however, takes into account the sparsity of the systems $\eqSys$, induced by the structure of the graph, and reduces the number of solution paths to be tracked so that the problem can be solved in practice. We present a detailed runtime analysis in Sec.~\ref{subsec:runtime}.
Note that neither the structure of \eqref{eq:SetOfEq} nor the number of complex solutions in $V(F)$ does change if the underlying structure of the graph is the same~\cite{chen2015network,chen2016network}; but, depending on the potentials, 
the number of solutions in $V_{\mathbb{R}_+}^{*}(\mathbf{F})$ may change.


Finally, the performance of BP-decoding for low error-correcting codes is analyzed by obtaining all fixed points with NPHC for the associated factor graph and by comparing them with the exact solutions int Sec.~\ref{subsec:ldpc}.
\subsection{Evaluation Criteria} \label{subsec:evaluationCriteria}
We apply \eqref{eq:marginals} to the positive real solutions to obtain the marginals and compare them with marginals obtained by an implementation of BP without damping \cite{mooij2010libdai}; we further compare these results with the exact marginals obtained by the junction tree algorithm~\cite{lauritzen-junction-tree}.
To evaluate the correctness of approximated marginals we average the mean squared error (MSE) over all $N$ nodes. For binary RVs we can apply symmetry properties of the probability mass function $P(x_i) = 1 - P(\bar{x}_i)$, so that
\begin{align}
 \text{MSE} = \frac{2}{N} \sum_{i=1}^{N}|P(x_i)-\tilde{P}(x_i)|^2. 
 \label{eq:mse}
\end{align}

The mean magnetization $\magnetization = \frac{1}{N} \sum_{i=1}^{N} m_i$ describes the response of the system to the field $\field{}$~\cite{mezard2009}, where we parametrize binary RVs by the magnetization $m_i =  P(X_i=1) -P(X_i=-1)$.
Note that the difference between the mean magnetization and the approximate mean magnetization $\tilde{\magnetization}$ is the sum over all marginal errors
\begin{align}
 \magnetization - \tilde{\magnetization} = \frac{2}{N} \sum_{i=1}^{N} P(x_i) - \tilde{P}(x_i).
\end{align}
It follows that the mean belief $\langle P(\mathbf{X}) \rangle$ (i.e., the expectation of binary RVs with $\mathbb{S}=\{0,1\}$ averaged over all nodes) relates to the mean magnetization by
$ \langle P(\mathbf{X}) \rangle = \frac{1}{N} \sum_{i=1}^N P(x_i) = \frac{1}{2}(\magnetization+1)$.

Combining properly weighted marginals can lead to accuracy-improvements~\cite{braunstein2005survey, srinivasa2016survey}.
As the NPHC method provides all fixed point solutions we combine multiple solutions to evaluate how much
a weighted combination of -- either all or just some -- marginals increases the accuracy.

Therefore we consider the partition function and evaluate: the fixed point maximizing $\zbApproximate$, a combination of all fixed points weighted by $\zbApproximate$, and a weighted combination of fixed points that are local maxima of $\zbApproximate$.
Therefore we obtain the approximated marginals $\tilde{P}(\mathbf{X} = \mathbf{x})$ for every fixed point according to~\eqref{eq:marginals} and~\eqref{eq:pairwise};
the associated approximate Bethe partition function $\zbApproximateP$ is obtained by~\eqref{eq:betheFreeEnergy} and~\eqref{eq:bethePartition}.
Let $\setOfMarginals{ALL} = \{(\tilde{P}(\mathbf{X} = \mathbf{x}),\zbApproximate):(\underline{\smash{\mu}}, \underline{\alpha}) \in V_{\mathbb{R}_+}^{*}(\mathbf{F})\}$ be the set of all marginals and the associated Bethe partition function (obtained at fixed points). Accordingly, we define $\setOfMarginals{STABLE}$ for locally stable solutions only.
Then, if multiple fixed points exist, the weighted marginals are combined so that
\begin{align}
 \tilde{P}_{\text{ALL}}(\mathbf{X}=\mathbf{x})= \frac{1}{\sum\limits_{\zbApproximate \in \setOfMarginals{ALL}} \zbApproximate} \sum_{(\tilde{P}(\mathbf{X} = \mathbf{x}),\zbApproximate) \in \setOfMarginals{ALL}} \!\!\!\!\! \tilde{P}(\mathbf{X}= \mathbf{x}) \cdot \tilde{\zb}.
 \label{eq:combine}
\end{align}
Further we consider the combination of stable fixed points $\tilde{P}_{\text{STABLE}}(\mathbf{X}=\mathbf{x})$ and the marginals that maximize the Bethe partition function:
\begin{align} 
  \tilde{P}_{\text{MAX}}(\mathbf{X}=\mathbf{x}) = \argmax_{\tilde{P}} \zbApproximateP.
\end{align}

\subsection{Grid Graph with Random Factors (Spin Glass)}  \label{subsec:gridGraphRandom}
Consider a grid graph of size $N = 3 \times 3$ (Fig. \ref{fig:3x3}) with randomly distributed parameters. All pairwise and local potentials are sampled uniformly; i.e., $(\coupling{i}{j},\field{i}) \sim \mathcal{U}(-K,K)$. The larger the support of the uniform distribution is, 
the more difficult the task of inference becomes.
For $K=3$ inference is sufficiently difficult (cf. \cite{sutton2012improved}). 

According to~\eqref{eq:NrEq} the system of equations consists of $M = 72$ equations and $72$ unknowns. 
More specifically, \eqref{eq:SetOfEq} consists of 24 linear (i.e., normalization constraints), 40 quadratic, and 8 cubic equations; the total degree bounds the number of solutions by $d_t = 1^{24} \cdot 2^{40} \cdot 3^8 = 7.2 \cdot10^{15}$. 
Tracking of such an amount of solution paths is not feasible in practice, even with a parallel implementation of the NPHC method. The system of equations in~\eqref{eq:SetOfEq}, however, is sparse. We can exploit this sparsity that is induced by the graph structure if we consider the BKK bound and reduce the computational complexity. 
The number of complex solutions for this graph is bounded by $\text{BKK}=608$. After creating a suitable start system the problem is straightforward to solve with the NPHC method. It actually turns out that the BKK bound is tight for all graphs considered.

In particular we evaluated 100 grid graphs with random factors:
on 99 graphs BP converged after at most $10^4$ iterations.
Although the grid graph has multiple loops and the constrained $\fb$ is not necessarily convex~\cite[Corr.2]{heskes2004uniqueness}, we observe that for all $100$ graphs NPHC obtains a unique positive real solution corresponding to a unique BP fixed point.

\subsection{Grid Graph with Uniform Factors} \label{subsec:gridGraphUniform}
We further analyze the convergence properties of BP on grid graphs of size $N = 3 \times 3$ (Fig.~\ref{fig:3x3}) 
with constant potentials among all nodes and edges; i.e., for all edges $\coupling{i}{j} = J$ and for all nodes $\theta_i = \theta$. We apply BP and NPHC for 1681 graphs in the parameter region 
$(J,\theta) \in \{-2,-1.9,\ldots,1.9,2\}$. 

\begin{table*}[t]
\renewcommand{\arraystretch}{1.3}
\caption{\textsc{MSE of Marginals and Combined Marginals Obtained by BP and NPHC for the Grid Graph with Uniform Factors.}}
\label{tab:3x3}
\centering
\begin{tabular}{l l l l l l l }
   \toprule
  \multicolumn{2}{c}{Parameters}&  \multicolumn{2}{c}{Fixed Points} & \multicolumn{3}{c}{Combined} \\ 
  \cmidrule(lr){1-2} \cmidrule(lr){3-4}  \cmidrule(lr){5-7}
  Couplings & Local Field & BP & NPHC & MAX & ALL & STABLE \\ \midrule
  $ J \in [-2,2]$ & $\theta \in [-2,2]$ & 0.197 & 0.016 & 0.037 & 0.005 & 0.004 \\  
  $ J \in (I)$ & $\theta \in (I)$ & 0.010 & 0.010 & 0.076 & $9.2\cdot10^{-4}$ & $1.0\cdot10^{-9}$\\ 
   $ J \in (II)$ & $\theta \in (II)$  & 0.836 & 0.050 & 0.126& 0.003 & $1.5\cdot10^{-6}$\\
  $ J \in (III)$ & $\theta \in (III)$ & 0.006 & 0.006 &  0.006 & 0.006 & 0.006\\ 
 \bottomrule
 \end{tabular}
\end{table*}

\begin{figure*}[t]
\centering
\subfloat[][]{ \label{fig:3x3-pos}
\centering
  \begin{tikzpicture} 
  \begin{axis}[width = 0.33\linewidth,   
               colormap/hot, point meta min=0, point meta max=3.5,
               view={0}{90}, 
               grid = major]
    \addplot3[surf, shader=flat corner, draw = gray]
    table{PosSolutions_3x3};
  \end{axis}
  \node at (2.22,-0.6) {$\theta$};
  \node at (-0.5,1.81) {$J$};
  \node at (2.25,3.1) {$(I)$};
  \node at (2.25,1.8) {$(III)$};
  \node at (2.25,0.4) {$(II)$};
  
  \end{tikzpicture}}
 \hspace{4cm}
\subfloat[][]{ \label{fig:3x3-real}
\centering
  \begin{tikzpicture}
  \begin{axis}[view/h=70,
               width = 0.33\linewidth,
	       colormap/hot]
    \addplot3[surf]
    table{RealSolutions_3x3};
  \end{axis}
   \node at (0,0.25) {$\theta$};
   \node at (3.2,-0.4) {$J$};
  \end{tikzpicture}}
  \caption[]{ Number of fixed points on the grid graph of size $N=3 \times 3$: \subref{fig:3x3-pos}
  number of positive real fixed points (yellow: unique fixed point, red: three fixed points); \subref{fig:3x3-real} number of real fixed points. The increase in number of both real solutions and positive real solutions, indicates a phase transition.}
   \label{fig:3x3-solutions}
\end{figure*}
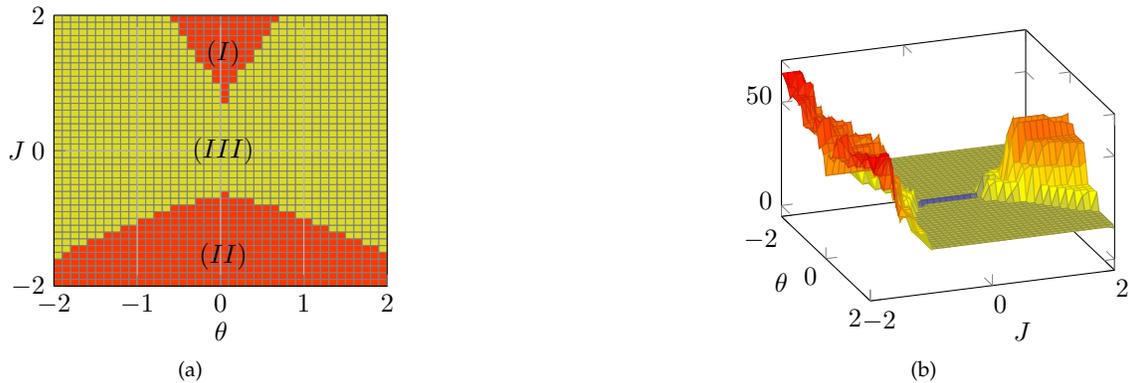

The number of solutions in $V_{\mathbb{R}_+}^{*}(\mathbf{F})$  is presented in Fig.~\ref{fig:3x3-pos}. In the well-behaved region $(III)$ of the parameter space a unique fixed point exists, whereas 3 fixed points exist in $(I)$ and $(II)$ -- this is in accordance with statistical mechanics\footnote{Note that the graph under consideration is of finite size and thus $|\neighbors{i}|$ varies among the nodes. As a consequence the partitioning according to \eqref{eq:phase1} - \eqref{eq:phase3}, are only approximations.} \cite[p.43]{mezard2009}.
Interestingly, we observe a close relation between the onset of phase transitions\footnote{We use the term phase transition for the finite-size manifestation of the phase transition in the thermodynamical sense.} and the increase in the number of real solutions in Fig.~\ref{fig:3x3-real}.
Most of these solutions, however, correspond to negative message values 
which violate Lemma~\ref{lm:msg} and are not feasible.

BP converges to some fixed point on all $1681$ graphs within at most $10^4$ iterations.
This raises the question: what is the approximation error of BP if it converges to the best possible fixed point?
Or, speaking in terms of free energies, how large is the gap between the global minimum of the constrained $\fb$ and the minimum of the Gibbs free energy?
To answer this question we evaluate the correctness of the approximated marginals by computing the MSE between the exact and the approximated marginals according to~\eqref{eq:mse}. The results are presented in Table~\ref{tab:3x3}.

Averaged over all graphs we can see that BP does not necessarily converge to the best solution. 
For NPHC we present the MSE for the fixed point with the lowest MSE; this highlights the existence of fixed points, which give \emph{more} accurate approximation than BP. 
Looking at all parameter regions separately we can see that BP does converge to the global optimum in $(III)$, as well as in $(I)$. 
In the antiferromagnetic region $(II)$ BP converges to a fixed point that does not necessarily give the best possible approximation.

If we consider regions with multiple fixed point solutions (i.e., $(I)$ and $(II)$) it becomes obvious that the fixed point maximizing $\zbApproximate$ is not necessarily the best one; 
this is especially surprising as BP obtains the best possible fixed point solution inside $(I)$. 
It turns out that for $\field{}=0$ identical initialized messages $\msg[0]{i}{j} = \mu^0$ correspond to a fixed point that, although potentially being unstable, equals the exact solution where $P(\bar{x}_i) = P(x_i)$ for all $X_i$ (cf. convergence results in Fig.~\ref{fig:convergence2x2}).

Inspired by these observations, 
one should not only consider the fixed point maximizing $\zbApproximate$, i.e., $\tilde{P}_{\text{MAX}}(\mathbf{X}=\mathbf{x})$, but rather obtain multiple fixed points by NPHC and combine them.
Indeed, especially inside region $(I)$ a combination of \emph{all} fixed point solutions according to~\eqref{eq:combine}, i.e., $\tilde{P}_{\text{ALL}}(\mathbf{X}=\mathbf{x})$ increases the accuracy of the approximation. If we combine the stable solution only, i.e., $\tilde{P}_{\text{STABLE}}(\mathbf{X}=\mathbf{x})$, the accuracy increases even more and gives the \emph{most} accurate approximation over the entire parameter space.

\subsection{Fully Connected Graph with Uniform Factors} \label{subsec:2x2}
We consider a fully connected Ising model with $|\mathbf{X}| = 2\times 2$ binary RVs (Fig.~\ref{fig:2x2}).
The system of equations~\eqref{eq:SetOfEq} consists of $36$ equations in $36$ unknowns and has its number of solutions bounded by the total degree $d_t=1^{12}\cdot2^{24}=16.8\cdot10^6$. Similar to Sec.~\ref{subsec:gridGraphRandom} 
a much tighter bound of $120$ is provided by the BKK bound.
Among all four nodes we apply uniform factors $\coupling{i}{j} = J$  and $\field{i} = \field{}$. This type of graph is particularly interesting because one can derive  exact conditions where phase transitions occur (cf. Section~\ref{subsec:freeEnergy}).


\begin{table*}[!t]

\renewcommand{\arraystretch}{1.3}
\caption{\textsc{MSE of Marginals and Combined Marginals Obtained by BP and NPHC for the Fully Connected Graph With Uniform Factors.}}
\label{tab:2x2}
\centering
 \begin{tabular}{l l l l l l l}
   \toprule
  \multicolumn{2}{c}{Parameters}&  \multicolumn{2}{c}{Fixed Points} & \multicolumn{3}{c}{Combined} \\ 
  \cmidrule(lr){1-2} \cmidrule(lr){3-4}  \cmidrule(lr){5-7}
  Couplings & Local Field & BP & NPHC & MAX & ALL & STABLE \\ \midrule
  $ J \in [-2,2]$ & $\theta \in [-2,2]$ & 0.069 & 0.007 & 0.011 &0.003 & 0.0027\\ 
  $ J \in (I)$ & $\theta \in (I)$       & 0.034 & 0.033 & 0.070 &0.002 & $2.0\cdot10^{-8}$ \\ 
  $ J \in (II)$ & $\theta \in (II)$     & 0.304 & 0.004 & 0.004 &0.004 &0.004\\ 
   $ J \in (III)$ & $\theta \in (III)$  & 0.003 & 0.003 & 0.003 &0.003 &0.003 \\ \bottomrule
 \end{tabular}
\end{table*}
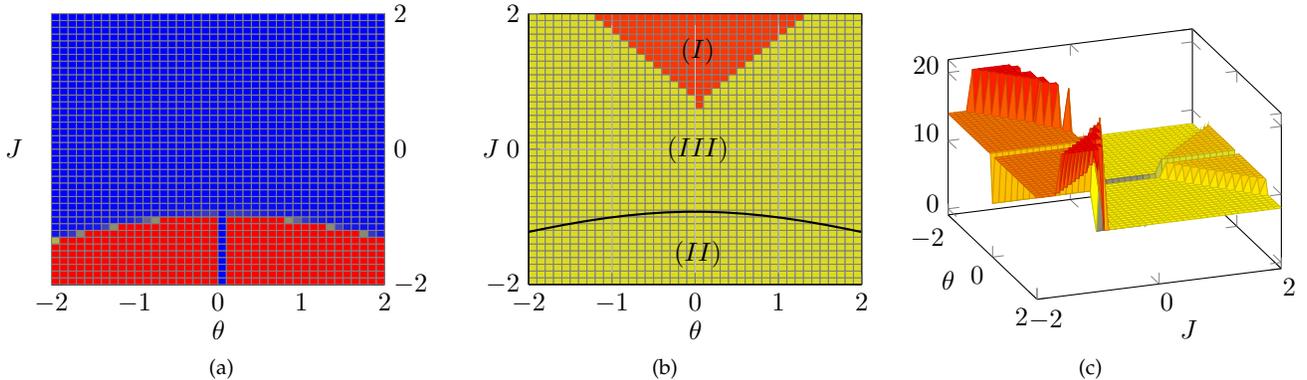
\begin{figure*}[!t]
\centering
\subfloat[][]{ \label{fig:convergence2x2}
  \centering
  \begin{tikzpicture}
  \begin{axis}
   [width = 0.33\linewidth,view={-270}{-90},
               colormap/hot, point meta min=0, point meta max=10000,
               grid = major]
    \addplot3[surf, shader=flat corner, draw = gray,xlabel style={rotate=-90}]
    table{iterations2x2};
  \end{axis};
  \node at (2.22,-0.6) {$\theta$};
  \node at (-0.5,1.81) {$J$};
  \end{tikzpicture}}
\subfloat[][] { \label{fig:solPos2x2}
\centering
  \begin{tikzpicture}[scale = 1]
  \begin{axis}[width = 0.33\linewidth,
               colormap/hot, point meta min=0, point meta max=3.5,
               view={0}{90}, 
               grid = major]
    \addplot3[surf, shader=flat corner, draw = gray]
    table{PosSolutions_2x2};
  \end{axis}
  \node at (2.22,-0.6) {$\theta$};
  \node at (-0.5,1.81) {$J$};
  \node at (2.25,3.1) {$(I)$};
  \node at (2.25,1.8) {$(III)$};
  \node at (2.25,0.4) {$(II)$};
  
  \node (L)  at (0,0.7) {};
  \node (R)  at (4.42,0.7) {}; 
  \path [black, thick](L.center) edge  [bend left=12] (R.center);
  \end{tikzpicture}}
\subfloat[][] { \label{fig:solReal2x2}
\centering
  \begin{tikzpicture}[scale = 1]
  \begin{axis}[width = 0.33\linewidth,
	       view/h=70,
	       colormap/hot]
    \addplot3[surf]
    table{RealSolutions_2x2};
  \end{axis}
   \node at (0,0.25) {$\theta$};
   \node at (3.2,-0.4) {$J$};
  \end{tikzpicture}}
\caption[]{Fully connected graph with $N=2\times2$. \subref{fig:convergence2x2} Convergence of BP: for the blue region BP did convergence -- for the red region it did not converge after $4\cdot 10^5$ iterations. \subref{fig:solPos2x2} Number of fixed points (yellow: unique fixed point, red: three fixed points. \subref{fig:solReal2x2} Number of real solutions -- note the sudden increase at the onset of phase transitions.}
\label{fig:2x2-solutions}
\end{figure*}

For $(J,\theta) \in (III)$ BP has a unique fixed point, which is a stable attractor in the whole message space \cite{mooij2005properties}. In $(I)$ three fixed points satisfy \eqref{eq:SetOfEq}, one of which is unstable and a local minimum of $\zbApproximate$. Both other fixed points are local maxima of $\zbApproximate$ and BP converges to one of them.
In the antiferromagnetic case BP does only converge inside $(III)$ and not inside $(II)$ as shown in Fig.~\ref{fig:convergence2x2}. We can see two interesting effects:
first, in Fig.~\ref{fig:solReal2x2} the number of real solutions increases at the onset of phase transitions; 
secondly, even though the convergence of BP breaks down at the phase transition (Fig.~\ref{fig:2x2-solutions}) a unique fixed  point exists inside $(II)$ (Fig.~\ref{fig:solPos2x2}) that gives an accurate approximation (cf. Table~\ref{tab:2x2}).

Similar as in Sec.~\ref{subsec:gridGraphUniform} we asses the MSE of the marginals obtained by NPHC and BP; averaged over all graphs, and for each distinct region  in Table~\ref{tab:2x2}. 
Indeed, inside region $(II)$ the marginals obtained by NPHC give a much better approximation than BP does.

Furthermore, note that the fixed point maximizing $\zbApproximate$ does not always give the best approximation. 
For $\field{}=0$ BP is initialized at the unstable fixed point and obtains the exact marginals (cf. Sec.~\ref{subsec:gridGraphUniform}). 
If multiple fixed points exist, a weighted combination of all marginals according to~\eqref{eq:combine} increases the accuracy -- only considering stable solutions, i.e., $\tilde{P}_{\text{STABLE}}(\mathbf{X}=\mathbf{x})$, gives the \emph{most} accurate approximations. 


\subsection{Fixed Point Evolution}\label{sec:evolution}
To get further insights we investigate the dependence of the accuracy of the fixed points on the parameters $(J,\field{})$ and compare the fixed point solutions obtained by NPHC to the exact solution.
Therefore, we fix the values of $\field{} \in \{0,0.1,0.5\}$ and vary $J\in[-2,2]$.
We illustrate the mean magnetization (cf. Sec.~\ref{subsec:evaluationCriteria}) of the exact solution and the approximate mean magnetization of all fixed points obtained by NPHC for both the grid graph and the fully connected graph in Fig.~\ref{fig:3x3-slice}-\ref{fig:2x2-slice}.

The exact solution (red)
is obtained by the junction tree algorithm~\cite{lauritzen-junction-tree}; solutions to~\eqref{eq:SetOfEq} are obtained by NPHC and are depicted by blue dots (stable) and by green dots (unstable). We further illustrate the fixed point that maximizes $\zbApproximate$, i.e., $\tilde{P}_{\text{MAX}}(\mathbf{X}=\mathbf{x})$, in black.

A fixed point is locally stable if a neighborhood exists such that messages inside this neighborhood converge to the fixed point~\cite[pp.170]{teschl2003} and locally unstable otherwise.
To analyze local stability one has to obtain all fixed points $(\underline{\smash{\mu}}, \underline{\alpha}) \in V_{\mathbb{R}_+}^{*}(\mathbf{F})$ with the NPHC method first.
Then BP is linearized in every fixed point by taking the partial derivatives of every messages with respect to all other messages, i.e., by analyzing the Jacobian matrix.
Finally, the eigenvalues of the Jacobian matrix determine the stability of the fixed point.
In practice one often uses damping, i.e., to replace  messages with a weighted average of older messages: this may change the stability of the fixed points, but does not change the solutions $V_{\mathbb{R}_+}^{*}$ of~\eqref{eq:SetOfEq} (cf. Sec.~\ref{subsec:beliefPropagation}).
We discuss some implications of the local stability analysis here, but focus mainly on the accuracy of the fixed points. For a detailed stability analysis of BP on Ising models we refer the reader to~\cite{mooij2005properties,knoll2017stability}.

\begin{figure*}[!t] 
\centering
\subfloat[][]{ \label{fig:3x3-1} 
         \begin{tikzpicture}
	\begin{axis}[width=0.7\columnwidth, xmin=-2,xmax=2,ymin=-1,ymax=1, ytick={-1,0,...,1},
	xlabel = $J$,
	ylabel = $\magnetization\text{ , }\tilde{\magnetization}$,
	mark size = 0.7pt,%
	minor xtick={-2,-1.5,...,2}, minor ytick={-1,-0.75,...,1},
	grid=both,
	scatter/classes={%
		e={mark=*,red},
		a={mark=*,cyan},
		u={mark=*,mygreen}}]
	\addplot[scatter,only marks, mark size=0.9,%
		scatter src=explicit symbolic]%
	table[meta=label] {3x3-th0};
	\addplot[black, thick] table {3x3-th0-bp-stable1};
	\addplot[black, thick] table {3x3-th0-bp-stable2};
	\addplot[black, thick] table {3x3-th0-bp-stable3};
	\addplot[black, thick] table {3x3-th0-bp-stable4};
	\addplot[black, thick] table {3x3-th0-bp-stable5};
	\addplot[mygreen] table {3x3-th0-bp-instable1}; 
	\addplot[mygreen,smooth] table {3x3-th0-bp-instable2};
	\addplot[red,thick] table {3x3-th0-exact};
	\addplot +[black, densely dashed, mark=none] coordinates {(-0.66,-1) (-0.66, 1)};
	\addplot +[black, densely dashed, mark=none] coordinates {(0.66, -1) (0.66, 1)}; 
	\end{axis}
	\node at (0.6,0.65) {$(II)$};
	\node at (2.3,0.65) {$(III)$};
	\node at (4,0.65) {$(I)$};
\end{tikzpicture}} 
\subfloat[][]{ \label{fig:3x3-2}         \begin{tikzpicture}
	\begin{axis}[width=0.7\columnwidth, xmin=-2,xmax=2,ymin=-1,ymax=1,ytick={-1,0,...,1},
	xlabel = $J$,
	mark size = 0.7pt,%
	minor xtick={-2,-1.5,...,2}, minor ytick={-1,-0.75,...,1},
	grid=both,
	scatter/classes={%
		e={mark=*,red},
		f={mark=*,cyan},
		u={mark=*,mygreen}}]
	\addplot[scatter,only marks, mark size=0.9,
		scatter src=explicit symbolic]%
	table[meta=label] {3x3-th0p1};
	\addplot[black, thick] table {3x3-th0p1-bp-stable1};
	\addplot[cyan] table {3x3-th0p1-bp-stable2};
	\addplot[cyan] table {3x3-th0p1-bp-stable3};
	\addplot[mygreen] table {3x3-th0p1-bp-instable1};
	\addplot[mygreen] table {3x3-th0p1-bp-instable2};
	\addplot[red,thick] table {3x3-th0p1-exact};
	\addplot +[black,mark=none] coordinates {(0.96, -1) (0.96, 1)};
	\addplot +[black,mark=none] coordinates {(-0.72, -1) (-0.72, 1)};
	\end{axis}
	\node at (0.6,0.65) {$(II)$};
	\node at (2.3,0.65) {$(III)$};
	\node at (4,0.65) {$(I)$};

\end{tikzpicture}}
\subfloat[][]{ \label{fig:3x3-3}         \begin{tikzpicture}
	\begin{axis}[width=0.7\columnwidth, xmin=-2,xmax=2,ymin=-1,ymax=1,ytick={-1,0,...,1},
	xlabel = $J$,
	mark size = 0.7pt,%
	minor xtick={-2,-1.5,...,2}, minor ytick={-1,-0.75,...,1},
	grid=both,
	scatter/classes={%
		e={mark=*,red},
		f={mark=*,cyan},
		u={mark=*,mygreen}}]
	\addplot[scatter,only marks, mark size=0.9,
		scatter src=explicit symbolic]%
	table[meta=label] {3x3-th0p5};
	\addplot[black, thick] table {3x3-th0p5-bp-stable1};
	\addplot[cyan] table {3x3-th0p5-bp-stable2};
	\addplot[cyan] table {3x3-th0p5-bp-stable3};
	\addplot[mygreen] table {3x3-th0p5-instable1};
	\addplot[mygreen] table {3x3-th0p5-instable2};
	\addplot[red,thick] table {3x3-th0p5-exact};
	\addplot +[black,mark=none] coordinates {(1.68, -1) (1.68, 1)};
	\addplot +[black,mark=none] coordinates {(-0.81, -1) (-0.81, 1)};
	\end{axis}
	\node at (0.6,0.65) {$(II)$};
	\node at (2.3,0.65) {$(III)$};
	\node at (4.45,0.65) {$(I)$};
\end{tikzpicture}}  
\caption[]{Results for the grid graph of size $N=3\times3$; mean magnetization $\magnetization$ and $\tilde{\magnetization}$ for $J \in [-2,2]$  and for: \subref{fig:3x3-1} $\field{}=0$,  \subref{fig:3x3-2} $\field{}=0.1$, and \subref{fig:3x3-3} $\field{}=0.5$. The exact solution is illustrated in red. All fixed points obtained by NPHC are depicted by blue dots (stable) and by green dots (unstable). The fixed point maximizing the partition function is illustrated in black.}
\label{fig:3x3-slice}
\end{figure*}

\begin{figure*}[!t] 
\centering
\subfloat[][]{ \label{fig:2x2-1} 
         \begin{tikzpicture}
	\begin{axis}[width=0.7\columnwidth, xmin=-2,xmax=2,ymin=-1,ymax=1, ytick={-1,0,...,1},
	xlabel = $J$,
	ylabel = $\magnetization\text{ , }\tilde{\magnetization}$,
	mark size = 0.7pt,%
	minor xtick={-2,-1.5,...,2}, minor ytick={-1,-0.75,...,1},
	grid=both,
	scatter/classes={%
		e={mark=*,red},
		f={mark=*,cyan},
		d={mark=*,mygreen},
		u={mark=*,mygreen}}]
	\addplot[scatter,only marks, mark size=0.9,
		scatter src=explicit symbolic]%
	table[meta=label] {2x2th0};
	\addplot[black, thick] table {2x2-th0-bp-stable1};
	\addplot[black, thick, smooth] table {2x2-th0-bp-stable2};
	\addplot[black, thick, smooth] table {2x2-th0-bp-stable3};
	\addplot[mygreen] table {2x2-th0-bp-instable};
	\addplot[red,thick] table {2x2-th0-exact};
	\addplot +[black,mark=none] coordinates {(0.55, -1) (0.55, 1)};
	\addplot +[black ,mark=none] coordinates {(-0.55, -1) (-0.55, 1)};
	\end{axis}
	\node at (0.6,0.65) {$(II)$};
	\node at (2.3,0.65) {$(III)$};
	\node at (4,0.65) {$(I)$};
\end{tikzpicture}} 
\subfloat[][]{ \label{fig:2x2-2}         \begin{tikzpicture}
	\begin{axis}[width=0.7\columnwidth, xmin=-2,xmax=2,ymin=-1,ymax=1,ytick={-1,0,...,1},
	xlabel = $J$,
	mark size = 0.7pt,%
	minor xtick={-2,-1.5,...,2}, minor ytick={-1,-0.75,...,1},
	grid=both,
	scatter/classes={%
		e={mark=*,red},
		f={mark=*,cyan},
		d={mark=*,mygreen},
		u={mark=*,mygreen}}]
	\addplot[scatter,only marks, mark size=0.9,
		scatter src=explicit symbolic]%
	table[meta=label] {2x2-th0p1};
	\addplot[black,thick] table {2x2-th0p1-bp-stable1};
	\addplot[cyan,smooth] table {2x2-th0p1-bp-stable2};
	\addplot[mygreen] table {2x2-th0p1-bp-instable};
	\addplot[red,thick] table {2x2-th0p1-exact};
	\addplot +[black,mark=none] coordinates {(0.76, -1) (0.76, 1)};
	\addplot +[black ,mark=none] coordinates {(-1.02, -1) (-1.02, 1)};
	\end{axis}
	\node at (0.6,0.65) {$(II)$};
	\node at (2.3,0.65) {$(III)$};
	\node at (4,0.65) {$(I)$};
\end{tikzpicture}}
\subfloat[][]{ \label{fig:2x2-3}         \begin{tikzpicture}
	\begin{axis}[width=0.7\columnwidth, xmin=-2,xmax=2,ymin=-1,ymax=1,ytick={-1,0,...,1},
	xlabel = $J$,
	mark size = 0.7pt,%
	minor xtick={-2,-1.5,...,2}, minor ytick={-1,-0.75,...,1},
	grid=both,
	scatter/classes={%
		e={mark=*,red},
		f={mark=*,cyan},
		d={mark=*,mygreen},
		u={mark=*,mygreen}}]
	\addplot[scatter,only marks, mark size=0.9,
		scatter src=explicit symbolic]%
	table[meta=label] {2x2-th0p5};
	\addplot[black,thick] table {2x2-th0p5-bp-stable1};
	\addplot[cyan] table {2x2-th0p5-bp-stable2};
	\addplot[mygreen] table {2x2-th0p5-bp-instable};
	\addplot[red,thick] table {2x2-th0p5-exact};
	\addplot +[black,mark=none] coordinates {(1.18, -1) (1.18, 1)};
	\addplot +[black ,mark=none] coordinates {(-1.06, -1) (-1.06, 1)};
	\end{axis}
	\node at (0.6,0.65) {$(II)$};
	\node at (2.3,0.65) {$(III)$};
	\node at (4,0.65) {$(I)$};
\end{tikzpicture}} 
\caption[]{Results for the fully connected graph of size $N=2\times2$; mean magnetization $\magnetization$ and $\tilde{\magnetization}$ for $J \in [-2,2]$ and for: \subref{fig:2x2-1} $\field{}=0$,  \subref{fig:2x2-2} $\field{}=0.1$, and \subref{fig:2x2-3} $\field{}=0.5$. The exact solution is illustrated in red. All fixed points obtained by NPHC are depicted by blue dots (stable) and by green dots (unstable). The fixed point maximizing the partition function is illustrated in black.}
\label{fig:2x2-slice}
\end{figure*}

The worst case scenario in terms of accuracy is shown in Fig.~\ref{fig:3x3-1} for the grid graph with $N=9$ binary RVs. For $\tilde{\magnetization} = 0$ a fixed point exists that equals the exact solution.
This fixed point, however, 
is only stable for $(J,\field{}) \in (III)$. As $|J|$ increases to the onset of a phase transition
this fixed point becomes unstable and
two additional solutions emerge. These additional solutions are symmetric, stable, and guarantee the convergence of BP on this graph (Fig.~\ref{fig:3x3-1}).

For $\field{} \neq 0$ 
a unique stable fixed point exists inside $(III)$. If we gradually increase $J$ until $(J,\field{}) \in (I)$ two additional fixed points emerge, one of which is unstable (Fig.~\ref{fig:3x3-2}).
Note that the fixed point maximizing $\zbApproximate$ (black) remains stable for all values of $J \in [-2,2]$.
An increase in $\field{}$ (Fig.~\ref{fig:3x3-3}) does enlarge the region where a unique fixed point exists (Fig.~\ref{fig:3x3-pos}) and further increases the accuracy of the fixed point maximizing $\zbApproximate$.
For $(J,\field{})\in(II)$ similar behavior is observed: i.e., for small values of $\field{}$ the unstable fixed point has the highest accuracy, but as $\field{}$ increases, the accuracy of the fixed point maximizing $\zbApproximate$ increases as well.


For the fully connected graph with $N = 4$ binary RVs (Fig.~\ref{fig:2x2}) results are shown in Fig.~\ref{fig:2x2-slice}.
For $\field{}=0$ and large values of $J$ the fixed point with $\tilde{\magnetization}=0$ is unstable and is accompanied by two symmetric, stable fixed points (Fig.~\ref{fig:2x2-1}). 
In contrast to the grid graph a unique fixed point exists for $(J,\field{}) \in (II)$; this fixed point, however, is unstable (Fig.~\ref{fig:solPos2x2}).
Therefore, despite the existence of an accurate fixed point, BP does not converge (Fig.~\ref{fig:convergence2x2}). 

For $\field{}\neq 0$ the non-convergent region $(II)$ is reduced, but the problem of a unique unstable fixed point persists. 
In contrast to the grid graph, the fully connected graph allows for frustrations with purely antiferromagnetic interactions. This points at a close connection between the existence of frustrations and the existence of a unique unstable solution.
For $(J,\field{}) \in (I)$ the fully connected graph behaves similar to the grid graph, i.e, the accuracy of the fixed point maximizing $\zbApproximate$ increases as $\field{}$ increases (Fig.~\ref{fig:2x2-2}-\ref{fig:2x2-3}).


Our main findings are:
First, increasing the field $\theta$ does lead to better convergence properties (cf.~\cite{knoll2017stability}) and increases the accuracy of the fixed point maximizing $\zbApproximate$.
Secondly, for $\field{} \neq 0$ the fixed point maximizing $\zbApproximate$ is unique and varies continuously under a change of $J$.
Finally, the stable fixed points are close to being symmetric, i.e., $\tilde{P}_{\text{stable,1}}(\mathbf{X}=\mathbf{x}) \cong \tilde{P}_{\text{stable,2}}(\mathbf{X}=\mathbf{\bar{x}})$. Consequently combining -- either all or only stable -- fixed points 
will not lead to good approximations unless a proper weighting by $\zbApproximate$ is applied. Applying a proper weighting, however, leads to accurate approximations (cf. Table~\ref{tab:3x3}-~\ref{tab:2x2}).
\subsection{Runtime Analysis}\label{subsec:runtime}
The time for solving~\eqref{eq:SetOfEq} is presented in Table~\ref{tab:timing} for grid graphs with random factors (Sec.~\ref{subsec:gridGraphRandom}), grid graphs with uniform factors (Sec.~\ref{subsec:gridGraphUniform}), and fully connected graph with uniform factors (Sec.~\ref{subsec:2x2}). Comparing the overall computation time of the NPHC method to BP it becomes obvious 
that NPHC
is no alternative in terms of computational efficiency. It is however the only method which is guaranteed to obtain \emph{all} fixed point solutions -- we were not able to apply the Gr\"obner basis method beyond a single-cycle graph with $N=4$. 
For the experiments we utilized a cluster-system with 160 CPUs: if we compare the overall computation time to the actual computation time utilizing the parallel implementation it becomes obvious that NPHC benefits tremendously from the high degree of parallelization. 
The runtime of BP depends mainly on the number of iterations and less on the size of the graph. Consequently, the stability of fixed points directly effects the performance of BP (cf. non-convergent region in Fig.~\ref{fig:convergence2x2}).  The NPHC method is much less sensitive to the stability of fixed point solutions; the mixed volume computation, which has the largest influence on the overall runtime, rather depends on the number of variables in~\eqref{eq:SetOfEq}. Note the mixed volume computation does not depend on the parameters. If one is interested in the fixed points for different parameter-sets on the same graph it would suffice to perform the mixed volume computation and start system creation only once; we did not do this to allow for a fair comparison.

\begin{table}[t!]
\renewcommand{\arraystretch}{1}
\setlength{\tabcolsep}{4pt}
\caption{\textsc{Runtime Comparison between BP and NPHC (in Seconds).}}
  \label{tab:timing}
  \centering
    \begin{tabular}{@{}l l l l l l l@{}}
      \toprule
      & \multicolumn{2}{c}{Grid Graph:}&  \multicolumn{2}{c}{Grid Graph:} &\multicolumn{2}{c}{Fully Conn. Graph:}\\
      & \multicolumn{2}{c}{Random $(J,\field{})$}&  \multicolumn{2}{c}{Uniform $(J,\field{})$} &\multicolumn{2}{c}{Uniform $(J,\field{})$} \\
	\cmidrule(lr){2-3}                  \cmidrule(lr){4-5}                \cmidrule(lr){6-7}
			& total         & parallel   & total        & parallel     & total & parallel                \\ \midrule
	Mixed Vol.    & 1364.5        & 11.2 	&      1311.1  & 12.0         & 0.27  & -- \\ 
	Path Track.   & 69.0          & 0.97 	&    70.5      & 1.0          & 3.77  & --\\ 
	Post Proc. & 29.9  	   & 1.40 	&    43.6      & 4.1          & 2.3   & -- \\ \midrule
	NPHC            & \multicolumn{2}{l}{13.57}  	   &   \multicolumn{2}{l}{17.1}  & \multicolumn{2}{l}{6.34}\\
	BP              & \multicolumn{2}{l}{$3.6\cdot10^{-3}$} &  \multicolumn{2}{l}{$0.7\cdot10^{-3}$}  & \multicolumn{2}{l}{0.03}\\
	\bottomrule
    \end{tabular}
\end{table}
\setlength{\tabcolsep}{6pt}

%
%
%

\subsection{Error-Correcting Codes}\label{subsec:ldpc}
One of the most prominent application areas where BP has a rich history of successful applications on cyclic graphs is iterative decoding.
We keep this section as self-contained as the scope of this paper allows. For a thorough introduction we refer the interested reader to the textbooks~\cite{mackay2003information, wymeersch2007iterative}; the connection between BP and decoding is further explained in great detail in~\cite{kschischang1998iterative,aji2000generalized,kschischang2001factor}.

We consider a binary symmetric channel (BSC) with a binary input $X_i \in \{0,1\}$ and a binary output $Y_i \in \{0,1\}$. The channel is specified by the error-probability $\epsilon$, where transmitted bits are flipped with probability $\epsilon$. That is $P(X_i=x_i|Y_i=y_i) = P(X_i=\bar{x_i}|Y_i=\bar{y_i}) = 1 - \epsilon$  and  $P(X_i=x_i|Y_i=\bar{y_i}) = P(X_i=\bar{x_i}|Y_i=y_i) =\epsilon$ (cf. Fig.~\ref{fig:bsc}). 
Additional, redundant bits help to detect and correct transmission errors.
The aim of error-correcting codes is to reach the desired error-correcting performance while introducing as little redundancy as necessary, i.e., to operate as close as possible to the theoretical limit.
Suppose we transmit a codeword with block length $N=7$ consisting of 4 source bits $X_1,\ldots, X_4$ and three parity-check bits $X_5,X_6,X_7$ that satisfy 
\begin{align}
 X_1\oplus X_2\oplus X_3 \oplus X_5 = 0, \\
 X_2\oplus X_3\oplus X_4 \oplus X_6 = 0,\\
 X_1\oplus X_3\oplus X_4 \oplus X_7 = 0,
\end{align}
where $\oplus$ is an XOR, i.e., the sum in modulo-2 arithmetic. This linear irregular code is the (7,4) Hamming code~\cite[Ch.1]{mackay2003information}.

In this example we assume that the sent message is $\mathbf{X} = (0,0,0,0,0,0,0)$\footnote{Note that the properties of the BSC are independent of $\mathbf{X}$.} and that exactly one bit suffers from a bit flip.
For irregular codes the degree of the variables $\neighborsVariable{Y}{i}$  varies; therefore, we consider two scenarios: either $\mathbf{Y} = (1,0,0,0,0,0,0)$ or $\mathbf{Y} = (0,0,0,0,0,1,0)$\footnote{Normally the performance of a code is studied over an ensemble of sent codewords where each bit flips with probability $\epsilon$.}.

\begin{figure}[t]
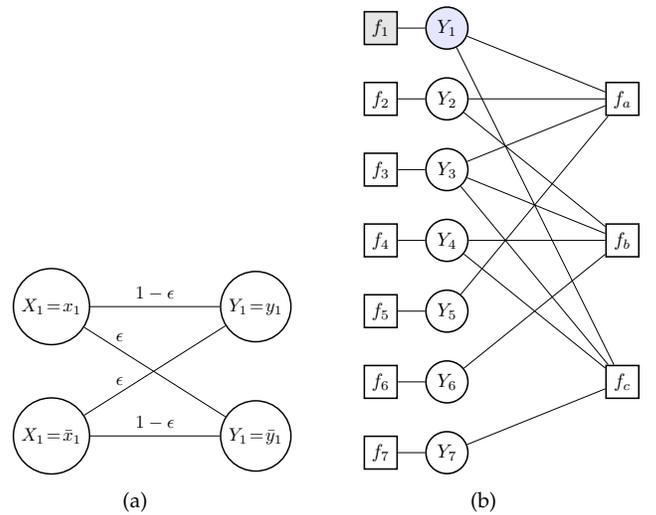

 \centering
\subfloat[][]{
\resizebox{0.5\columnwidth}{!}{
\input{BN-TIKZ.tex}

\begin{tikzpicture}[>=latex,text height=1.5ex,text depth=0.25ex]
  \matrix[row sep=0.5cm,column sep=0.5cm] 
  {     \node (Y_1) [hidden] {$X_1\!=\!x_1$}; & & & & &
        \node (Y_2) [hidden] {$Y_1\!=\!y_1$};\\ \\
	\node (Y_3) [hidden] {$X_1\!=\!\bar{x}_1$}; & & & & &
        \node (Y_4) [hidden] {$Y_1\!=\!\bar{y}_1$};\\
    };
    
    \path[-, anchor = south]        
        (Y_1) edge  node {$1-\epsilon$} (Y_2)
	(Y_3) edge  node {$1-\epsilon$} (Y_4)
	(Y_1) edge  node[pos=0.25] {$\epsilon$} (Y_4)
	(Y_3) edge  node [pos=0.25]{$\epsilon$} (Y_2)
	
       ;
\end{tikzpicture}

}
\label{fig:bsc}}
\subfloat[][]{
\resizebox{0.5\columnwidth}{!}{
\input{BN-TIKZ.tex}

\begin{tikzpicture}[>=latex,text height=1.5ex,text depth=0.25ex]
  \matrix[row sep=0.5cm,column sep=0.5cm] 
  {     \node (f_1) [factor] {$f_1$}; &
        \node (Y_1) [observed] {$\attribute{1}$};\\
        \node (f_2) [fo] {$f_2$}; &
        \node (Y_2) [hidden] {\attribute{2}};& & & & &
        \node (f_a) [fo] {$f_a$}; \\ 
        \node (f_3) [fo] {$f_3$}; &
	\node (Y_3) [hidden] {\attribute{3}}; \\
	\node (f_4) [fo] {$f_4$}; &
        \node (Y_4) [hidden] {\attribute{4}};& & & & &
        \node (f_b) [fo] {$f_b$}; \\ 
        \node (f_5) [fo] {$f_5$}; &
	\node (Y_5) [hidden] {\attribute{5}}; \\  
	\node (f_6) [fo] {$f_6$}; &
	\node (Y_6) [hidden] {\attribute{6}};& & & & &
	\node (f_c) [fo] {$f_c$};\\ 
	\node (f_7) [fo] {$f_7$}; &
	\node (Y_7) [hidden] {\attribute{7}};\\   
    };
    
    \path[-]        
        (Y_1) edge (f_a)
	(Y_1) edge (f_c)
	(Y_2) edge (f_a)
	(Y_2) edge (f_b)
	
	(Y_3) edge (f_a)
	(Y_3) edge (f_c)
	(Y_3) edge (f_b)
	(Y_4) edge (f_b)
	(Y_4) edge (f_c)
	
	(Y_5) edge (f_a)
	(Y_6) edge (f_b)
	(Y_7) edge (f_c)
	
	(f_1) edge (Y_1)
	(f_2) edge (Y_2)
	(f_3) edge (Y_3)
	(f_4) edge (Y_4)
	(f_5) edge (Y_5)
	(f_6) edge (Y_6)
	(f_7) edge (Y_7)

        ;
\end{tikzpicture}

}
\label{fig:fgSource}}
\caption{(a) Binary symmetric channel with error probability $\epsilon$; (b) Factor graph for the (7,4) Hamming code that corresponds to~\eqref{eq:ldpc} where $Y_1$ is flipped.}
\label{fig:bsc-fg}
\end{figure}

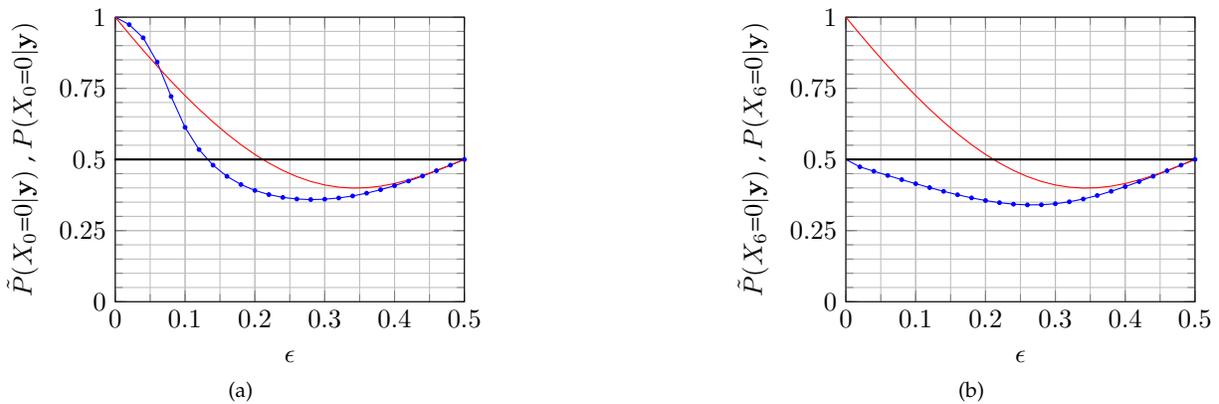
\begin{figure*}[t!]
 \centering
\subfloat[][]{ \label{fig:ldpc-source} 
         \begin{tikzpicture}
	\begin{axis}[width=0.7\columnwidth, xmin=0,xmax=0.5,ymin=0,ymax=1, ytick={0,0.25,...,1},xtick={0,0.1,...,1},
	mark size = 0.7pt,	
	xlabel = $\epsilon$,%
	ylabel = $\tilde{P}(X_0\text{=}0|\mathbf{y})\text{ , }P(X_0\text{=}0|\mathbf{y})$,
	minor xtick={0,0.05,...,1}, minor ytick={0,0.05,...,1},
	grid=both,
	scatter/classes={%
		e={mark=*,red},
		a={mark=*,blue}}]
	\addplot[scatter,only marks,%
		scatter src=explicit symbolic]%
	table[meta=label] {sourceBit};
	\addplot[blue] table {sourceBit-approx};
	\addplot[red] table {sourceBit-exact};
	\addplot[mark=none, black, thick] {0.5};
	\end{axis} \label{fig:ldpcResults-source}
\end{tikzpicture}} 
\hspace*{3cm}
\subfloat[][]{ \label{fig:ldpc-check}       
        \begin{tikzpicture}
	\begin{axis}[width=0.7\columnwidth, xmin=0,xmax=0.5,ymin=0,ymax=1, ytick={0,0.25,...,1},xtick={0,0.1,...,1},
	mark size = 0.7pt,
	xlabel = $\epsilon$,%
	ylabel = $\tilde{P}(X_6\text{=}0|\mathbf{y})\text{ , }P(X_6\text{=}0|\mathbf{y})$,
	minor xtick={0,0.05,...,1}, minor ytick={0,0.05,...,1},
	grid=both,
	scatter/classes={%
		e={mark=*,red},
		a={mark=*,blue}}]
	\addplot[scatter,only marks,%
		scatter src=explicit symbolic]%
	table[meta=label] {checkBit};
	\addplot[blue] table {checkBit-approx};
	\addplot[red] table {checkBit-exact};
	\addplot[mark=none, black, thick] {0.5};
	\end{axis} \label{fig:ldpcResults-check}
\end{tikzpicture}} 
\caption[]{Results for the (7,4) Hamming code. We compare the exact solution $P(X_i = 0|\mathbf{Y}=\mathbf{y})$ (red) to the approximate solution $\tilde{P}(X_i = 0|\mathbf{Y}=\mathbf{y})$ of NPHC (blue) as $\epsilon$ increases for: \subref{fig:ldpcResults-source} $Y_1$ is flipped and \subref{fig:ldpcResults-check} $Y_6$ is flipped.}
\label{fig:ldpc}
\end{figure*}

It is often convenient to express a code in factorized  form and represent it explicitly with a factor graph. A factor graph consists of variable nodes $Y_i$ and factor nodes $f_A$ where each factor $f_A$ acts as a function on all variables connected $\mathbf{Y_i}=\{Y_i \in \neighborsVariable{f}{A} \}$\footnote{Similar as in Sec.~\ref{subsec:beliefPropagation} we use $\neighborsVariable{\cdot}{}$ to specify the neighbors of nodes and variables.} . 
On a factor graph BP operates similar as introduced in Sec.~\ref{subsec:beliefPropagation}. Now two types of messages are sent along every edge: factor-to-variable messages $r_{f_A,Y_i}$ and  variable-to-factor messages $q_{Y_i,f_A}$. All messages are iteratively updated according to
\begin{align}
 r_{f_A,Y_i}^{n+1}(y_i) \!= \!\!\!\!\!\!\!\!\!\!\!\!\!\! \sum_{\mathbf{y_k} : Y_k \in \neighborsVariable{f}{A} \backslash \{Y_i \}} \!\!\!\!\!\!\!\!\!\!\!\!\!f_A(\mathbf{Y_k} = \mathbf{y_k}, Y_i = y_i) \!\!\!\!\!\!\!\!\!\prod_{Y_k \in \neighborsVariable{f}{A} \backslash \{Y_i\}} \!\!\!\!\!\!\!\!\!\!\! q_{Y_k,f_A}^{n}(y_k),
 \label{eq:factor-variable}
\end{align}
\begin{align}
 q_{Y_i,f_A}^{n+1}(y_i) = \alpha_{Y_i,f_A}^{n} \prod_{f_B \in \{\neighborsVariable{Y}{i} \backslash f_A\}} r_{f_B,Y_i}^{n}(y_i),
 \label{eq:variable-factor}
\end{align}
where $\alpha_{Y_i,f_A}$ is chosen such that $q_{Y_i,f_A}^{n+1}(y_i)+q_{Y_i,f_A}^{n+1}(\bar{y}_i) = 1$.
After all messages converged to a fixed point, the marginals of the variable nodes are approximated by the product of all incoming messages
\begin{align}
 P(Y_i=y_i) = \frac{1}{Z} \prod_{f_B \in \neighborsVariable{Y}{i}} r_{f_B,Y_i}^{n}(y_i).
\end{align}
Details of the factor graph representation can be found in~\cite[Ch.9]{mezard2009} and~\cite{loeliger2004introduction}.

Let us consider two types of factors: $f_i(Y_i) = P(X_i=x_i|Y_i=y_i)$ to model the BSC, and 
$f_A(\mathbf{Y_i})$ to verify if all parity-checks are satisfied. 
Then $f_A(\mathbf{Y_i}) = 1$ if the sum of all arguments $\sum_{\mathbf{Y_i}} y_i$ is even and $f_A(\mathbf{Y_i}) = 0 $ if the sum is odd.
The conditional probability for $\mathbf{X} = \mathbf{x}$ to be the codeword, given the received codeword $\mathbf{Y} = \mathbf{y}$ then is 
\begin{align}
\begin{split}
 P(\mathbf{X} =\mathbf{x}|&\mathbf{Y}= \mathbf{y}) = \frac{1}{Z} \prod_{i=1}^{7}f_i(Y_i) \cdot f_a(Y_1,Y_2,Y_3,Y_5) \cdot \\
  &f_b(Y_2,Y_3,Y_4,Y_6) \cdot f_c(Y_1,Y_3,Y_4,Y_7).
\end{split}
 \label{eq:ldpc}
\end{align}
The corresponding factor graph representation is shown in Fig~\ref{fig:fgSource}.

Now we 
create a system of equations similar as in Sec.~\ref{sec:eqSys} 
and obtain all fixed points with NPHC. 
A unique fixed point exists for all settings -- and this fixed point is stable, which justifies the application of BP on error-correcting codes.
We further estimate the accuracy of the approximation; this relates to the question: 
If we communicate over a BSC, how vulnerable is BP decoding to an increased error probability?
To answer this question we obtain the fixed points for $\epsilon \in [0,0.5]$ and compare
the exact solution obtained by the junction tree algorithm 
$P(X_i = 0|\mathbf{Y}=\mathbf{y})$ (red), to the approximate solution $\tilde{P}(X_i = 0|\mathbf{Y}=\mathbf{y})$ obtained by NPHC (blue) in 
Fig.~\ref{fig:ldpc}\footnote{Note for $\epsilon=0.5$ the transmission is random}.

An error can be corrected by BP decoding if $\tilde{P}(Y_i = 0|\mathbf{X}=\mathbf{x}) > 0.5$. With exact decoding a single bit-flip can be corrected for $\epsilon < 0.21$. 
According to the fixed points obtained by NPHC, however,
The fixed points obtained by NPHC however reveal that BP  
does not utilize the full potential of the code (Fig.~\ref{fig:ldpc}).
If $Y_1$ was corrupted the error can be corrected for $\epsilon < 0.13$ (Fig.~\ref{fig:ldpc-source}). BP fails to correct the error if $Y_6$ was flipped for all values of $\epsilon$ (Fig.~\ref{fig:ldpcResults-check}).
The reason therefore is a systematic error: the check-bit only has a single connection to the parity-check function $f_c$.
According to~\eqref{eq:variable-factor} $q_{Y_6,f_c}^{n+1}(y_i) = \alpha_{Y_7,f_7} \cdot r_{f_7,Y_7}(y_i)$; therefore $Y_6$ does not incorporate any information from the remaining graph.
To conclude, the higher the connectivity of a node, the more information of other bits is taken into account and the better the error-correction capability of BP and NPHC. 

\section{Conclusion}\label{sec:conclusion}

The NPHC method is presented as a tool to obtain all BP fixed point solutions. 
This work is an attempt to get a deeper understanding of BP, with potential implications for finding stronger conditions for uniqueness and convergence guarantees of BP fixed points. 

One key feature of our framework is to come up with a favorable upper bound on the number of solutions. In particular the BKK bound utilizes the sparsity of the polynomial system induced by the graph structure and is tight in all our experiments. 
Finally, to obtain all fixed points, we create an appropriate start system and track all solution paths in parallel.

On fully connected graphs and grid graphs with binary Ising factors we show how the number of fixed points evolves over a large parameter region. While, in practice, fixed points have to be positive, we empirically showed that there is a close relation between the occurrence of phase transitions and an increase in the number of real solutions. 

We empirically show an accuracy-gap between fixed points of BP and the best fixed points obtained by NPHC. In practice this justifies the exploration of multiple fixed points and selecting one that leads to the best approximation of the marginals. Moreover, we analyze the fixed point maximizing the partition function in detail: we show how it continuously deforms under varying parameters and how stability of this fixed point depends on the graph structure. The NPHC method further reveals for which parameters the fixed point maximizing the partition function does not correspond to the best fixed point obtained by NPHC. These observations motivate to consider a combination of weighted marginals. These weighted combinations provide strikingly accurate marginals whenever multiple fixed points exist.

When applied to graphs where BP does not converge, the NPHC method reveals that for some cases a unique fixed point does exist -- consequently, uniqueness of BP fixed points is by no means sufficient to guarantee convergence of BP. Further we conjecture a close connection between the existence of frustrations and the existence of a unique unstable fixed point.

One requirement of NPHC to be efficient is a tight upper bound on the number of solutions. The determination of this bound limits our current investigations to relatively small graphs. 
In future we aim to exploit the graph structure in order to reduce the complexity of identifying an upper bound on the number of solutions.

\ifCLASSOPTIONcompsoc
  \section*{Acknowledgments}
\else
  \section*{Acknowledgment}
\fi
Research supported in part by NSF under Grant DMS 11-15587. This work was supported by the Austrian Science Fund (FWF) under the project number P28070-N33.

\ifCLASSOPTIONcaptionsoff
  \newpage
\fi

 \bibliographystyle{IEEEtran}
 \bibliography{convergence_BP}

\end{document}